\documentclass{article}
\usepackage[margin=1in]{geometry}
\usepackage{float}
\usepackage{microtype}
\usepackage{natbib}
\usepackage{color}
\usepackage{xcolor}
\definecolor{mydarkblue}{rgb}{0,0.08,0.45}
\usepackage[hidelinks,colorlinks]{hyperref}
\hypersetup{
	colorlinks=true,
	linkcolor=mydarkblue,
	citecolor=mydarkblue,
	filecolor=mydarkblue,
	urlcolor=mydarkblue,
}
\usepackage[ruled]{algorithm2e}
\SetKwProg{Fn}{function}{}{end}
\SetKwProg{Set}{Setting: }{}{}
\SetKwProg{Alg}{Algorithm: }{}{}
\SetKwProg{FirstPass}{First Pass:}{}{}
\SetKwProg{SecondPass}{Second Pass:}{}{}
\SetKwProg{SVD}{SVD:}{}{}
\SetKwInput{Input}{Input}
\SetKw{KwTo}{in}
\SetKw{KwWhere}{where}

\usepackage{tikz}
\usetikzlibrary{automata,arrows}
\usepackage{titlesec}
\usepackage[utf8]{inputenc} % allow utf-8 input
\usepackage[T1]{fontenc}    % use 8-bit T1 fonts
\usepackage{hyperref}       % hyperlinks
\usepackage{url}            % simple URL typesetting
\usepackage{booktabs}       % professional-quality tables
\usepackage{amsfonts}       % blackboard math symbols
\usepackage{nicefrac}       % compact symbols for 1/2, etc.
\usepackage{microtype}      % microtypography
\usepackage{amsmath}
\usepackage{mathtools}
\usepackage[shortlabels]{enumitem}
\usepackage{amsthm}
\usepackage{amssymb}
\usepackage{mathrsfs}
\usepackage{wrapfig}
\usepackage{epsfig}
\usepackage{booktabs}
\usepackage[outdir=./]{epstopdf}
\usepackage{caption}
\usepackage[framemethod=TikZ]{mdframed}
\usepackage{lipsum}
\usepackage{subcaption}
\usepackage{multirow}
\usepackage{comment}
\usepackage{framed}
\usepackage{thmtools,thm-restate}
\usepackage{booktabs}
\usepackage{setspace}
\newtheorem{thm}{Theorem}

\newtheorem{lem}{Lemma}

\newtheorem{prop}{Proposition}
\newtheorem{definition}{Definition}
\newcommand*{\Prob}{\mathbb{P}}

\newcommand*{\rank}{\text{rank}}

\newcommand*{\Oh}{O}

\newcommand{\normsq}[1]{\| #1 \|_2^2}

\newcommand{\lpnorm}[2]{\| #1 \|_{#2}}

\title{Compressed Factorization: Fast and Accurate \\ Low-Rank Factorization of Compressively-Sensed Data}

\author{Vatsal Sharan \thanks{Equal contribution.} \hspace{20pt} Kai Sheng Tai \footnotemark[1] \hspace{20pt} Peter Bailis \hspace{20pt} Gregory Valiant \vspace{5pt}\\   \texttt{\{vsharan, kst, pbailis, valiant\}@cs.stanford.edu}\vspace{5pt} \\ Stanford University}

\begin{document}

\date{}
\maketitle
\begin{abstract}
What learning algorithms can be run directly on compressively-sensed data?
In this work, we consider the question of accurately and efficiently computing low-rank matrix or tensor factorizations given data compressed via random projections.  
We examine the approach of first performing factorization in the compressed domain, and then reconstructing the original high-dimensional factors from the recovered (compressed) factors.  
In both the matrix and tensor settings, we establish conditions under which this natural approach will provably recover the original factors.  
While it is well-known that random projections preserve a number of geometric properties of a dataset, our work can be viewed as showing that they can also preserve certain solutions of non-convex, NP-Hard problems like non-negative matrix factorization.
We support these theoretical results with experiments on synthetic data and demonstrate the practical applicability of compressed factorization on real-world gene expression and EEG time series datasets.
\end{abstract}

\section{Introduction}

We consider the setting where we are given data that has been compressed via random projections. 
This setting frequently arises when data is acquired via compressive measurements \citep{donoho2006compressed,candes2008introduction}, or when high-dimensional data is projected to lower dimension in order to reduce storage and bandwidth costs \citep{haupt2008compressed,abdulghani2012compressive}. 
In the former case, the use of compressive measurement enables higher throughput in signal acquisition, more compact sensors, and reduced data storage costs \citep{duarte2008single,candes2008introduction}.  
In the latter,  the use of random projections underlies many sketching algorithms for stream processing and distributed data processing applications \citep{cormode2012synopses}.

Due to the computational benefits of working directly in the compressed domain, there has been significant interest in understanding which learning tasks can be performed on compressed data. 
For example, consider the problem of supervised learning on data that is acquired via compressive measurements. 
\citet{calderbank2009compressed} show that it is possible to learn a linear classifier directly on the compressively sensed data with small loss in accuracy, hence avoiding the computational cost of first performing sparse recovery for each input prior to classification.
The problem of learning from compressed data has also been considered for several other learning tasks, such as linear discriminant analysis \citep{durrant2010compressed}, PCA \citep{fowler2009compressive,zhou2011godec,ha2015robust}, and regression \citep{zhou2009compressed,maillard2009compressed,kaban2014new}.

%Building off this line of work, we consider recovering low-rank matrix and tensor factorizations from compressed data.  In particular, our goal is to recover the factors in their original (uncompressed) domain.  We focus on the practically relevant setting where the factors are sparse, and consider a variety of problems in this setting, including sparse PCA, nonnegative matrix factorization (NMF), and tensor decomposition. 

%Building off this line of work, we consider low-rank matrix and tensor factorization on compressed data. We focus in particular on the setting where the factors are sparse---this includes commonly-used formulations such as sparse PCA and nonnegative matrix factorization (NMF). 

Building off this line of work, we consider the problem of performing low-rank matrix and tensor factorizations directly on compressed data, with the goal of recovering the low-rank factors in the original, uncompressed domain.
%In particular, our goal is to recover the factors in the original, uncompressed domain under the assumption that the factors are \emph{sparse}. 
Our results are thus relevant to a variety of problems in this setting, including sparse PCA, nonnegative matrix factorization (NMF), and Candecomp/Parafac (CP) tensor decomposition. 
As is standard in compressive sensing, we assume prior knowledge that the underlying factors are \emph{sparse}.

For clarity of exposition, we begin with the matrix factorization setting.
Consider a high-dimensional data matrix $M\in\mathbb{R}^{n \times m}$ that has a rank-$r$ factorization $M = WH$, where $W\in\mathbb{R}^{n\times r}$, $H \in \mathbb{R}^{r \times m}$, and $W$ is sparse. 
We are given the compressed measurements $\tilde{M}=PM$ for a known \emph{measurement matrix} $P\in\mathbb{R}^{d\times n}$, where $d < n$.
Our goal is to approximately recover the original factors $W$ and $H$ given the compressed data $\tilde{M}$ as accurately and efficiently as possible.
\begin{figure}
\centering
\includegraphics[width=0.5\linewidth]{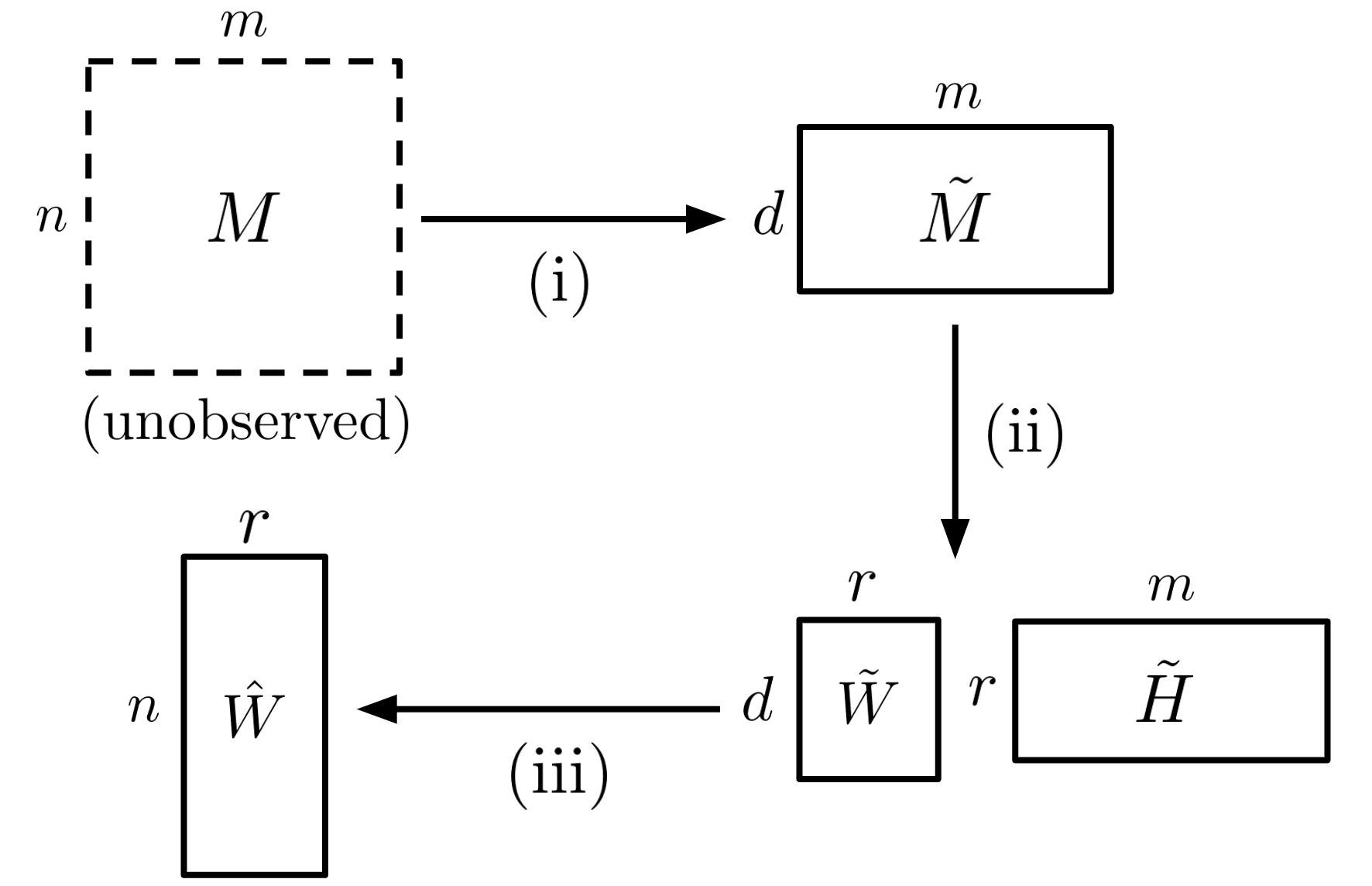}
\caption{Schematic illustration of compressed matrix factorization. (i) The matrix $\tilde{M}$ is a compressed version of the full data matrix $M$. (ii) We directly factorize $\tilde{M}$ to obtain matrices $\tilde{W}$ and $\tilde{H}$. (iii) Finally, we approximate the left factor of $M$ via sparse recovery on each column of $\tilde{W}$.}
\label{fig:overview}
\end{figure}
This setting of compressed data with sparse factors arises in a number of important practical domains.
For example, gene expression levels in a collection of tissue samples can be clustered using NMF to reveal correlations between particular genes and tissue types~\citep{gao2005improving}.
Since gene expression levels in each tissue sample are typically sparse, compressive sensing can be used to achieve more efficient measurement of the expression levels of large numbers of genes in each sample~\citep{parvaresh2008recovering}.
In this setting, each column of the $d\times m$ input matrix $\tilde{M}$ corresponds to the compressed measurements for the $m$ tissue samples, while each column of the matrix $W$ in the desired rank-$r$ factorization corresponds to the pattern of gene expression in each of the $r$ clusters.

%To give two motivating examples, first consider the task of analyzing gene expression data.
%The gene expression levels are typically sparse, and this is a domain where compressive sensing has been successfully applied for efficient data acquisition \citep{parvaresh2008recovering}. 
%Additionally, NMF has been widely used in this domain (on uncompressed data) to cluster tissue samples corresponding to different forms of cancer, and yields sparse factors \citep{gao2005improving}. 
%Another relevant setting is  EEG data where compressive sensing has been useful for bandwidth-limited data transmission \citep{zhang2013compressed}. 
%Tensor decomposition is commonly used in this domain to identify distinct modes of brain activity \cite{cong2015tensor}, and recovery of the factors in the original domain is essential for the interpretability of these factors.

%wtensor decomposition (TD) has been used to analyze EEG data where  in order to identify distinct modes of brain activity \cite{cong2015tensor}. In both these domains, compressed sensing has been applied for efficient data acquisition \citep{parvaresh2008recovering} and bandwidth-limited data transmission \citep{zhang2013compressed}.

We consider the natural approach of performing matrix factorization directly in the compressed domain (Fig.~\ref{fig:overview}): first factorize the compressed matrix $\tilde{M}$ to obtain factors $\tilde{W}$ and $\tilde{H}$, and then approximately recover each column of $W$ from the columns of $\tilde{W}$ using a sparse recovery algorithm that leverages the sparsity of the factors.
We refer to this ``compressed factorization'' method as \textsc{Factorize-Recover}. 
This approach has clear computational benefits over the alternative \textsc{Recover-Factorize} method of first recovering the matrix $M$ from the compressed measurements, and then performing low-rank factorization on the recovered matrix.
In particular, \textsc{Factorize-Recover} requires only $r$ calls to the sparse recovery algorithm, in contrast to $m \gg r$ calls for the alternative. 
This difference is significant in practice, \emph{e.g.} when $m$ is the number of samples and $r$ is a small constant.
Furthermore, we demonstrate empirically that \textsc{Factorize-Recover} also achieves better recovery error in practice on several real-world datasets.

%The following is a natural approach to directly recover the original factors $W$ and $H$ from the compressed data $\tilde{M}$. Hypothetically, say that factorization of the compressed matrix $\tilde{M}=\tilde{W}\tilde{H}$ yields the factorization $\tilde{W}=PW$ and $\tilde{H}=H$. As $\tilde{M}=PM$ and $M=WH$, this is certainly a valid factorization of the data. If we do recover this factorization, we can subsequently estimate $W$ from $\tilde{W}$ using a sparse recovery algorithm that leverages the sparsity of the factors. We refer to this approach of first factorizing the compressed matrix $\tilde{M}$ and then using sparse recovery to recover the original factors from the compressed factor as \textsc{Factorize-Recover}. \textsc{Factorize-Recover} has clear computational benefits over the alternative naive approach of first 
%recovering each column of $M$, and only then performing low-rank factorization on the recovered matrix.
%In particular, \textsc{Factorize-Recover} requires only $r$ calls to the sparse recovery algorithm, in contrast to $m \gg r$ calls for the alternative: this difference is significant in practice, for example when $m$ is the number of samples in the dataset and $r$ is a small constant. As $\tilde{M}=PM$ and $M=WH$, $\tilde{M}=(PW)H$ is a valid factorization of the data, but matrix factorizations are not unique in general. 
Note that the \textsc{Factorize-Recover} approach is guaranteed to work if the factorization of the compressed matrix $\tilde{M}$ yields the factors $\tilde{W}=PW$ and $\tilde{H}=H$, since we assume that the columns of $W$ are sparse and hence can be recovered from the columns of $\tilde{W}$ using sparse recovery.
%Note that if factorization of the compressed matrix $\tilde{M}$ yields the factors $\tilde{W}=PW$ and $\tilde{H}=H$, then the \textsc{Factorize-Recover} approach is guaranteed to work because the columns of $W$ are assumed to be sparse, and hence can be recovered from the columns of $\tilde{W}$ using sparse recovery. 
Thus, the success of the \textsc{Factorize-Recover} approach depends on finding this particular factorization of $\tilde{M}$.
Since matrix factorizations are not unique in general, we ask:
\emph{under what conditions is it possible to recover the ``correct'' factorization $\tilde{M}=(PW)H$ of the compressed data, from which the original factors can be successfully recovered?}\\

\noindent\textbf{Contributions.\hspace{0.5em}} In this work, we establish conditions under which \textsc{Factorize-Recover} provably succeeds, in both the matrix and tensor factorization domains.  
We complement our theoretical results with experimental validation that demonstrates both the accuracy of the recovered factors, as well as the computational speedup resulting from \textsc{Factorize-Recover} versus the alternative approach of first recovering the data in the original uncompressed domain, and then factorizing the result. 

Our main theoretical guarantee for sparse matrix factorizations, formally stated in Section~\ref{sec:nmf_theory}, provides a simple condition under which the factors of the compressed data are the compressed factors.  While the result is intuitive, the proof is delicate, and involves characterizing the likely sparsity of linear combinations of sparse vectors, exploiting graph theoretic properties of expander graphs. The crucial challenge in the proof is that the columns of $W$ get mixed after projection, and we need to argue that they are still the sparsest vectors in any possible factorization after projection. This mixing of the entries, and the need to argue about the uniqueness of factorizations after projection, makes our setup significantly more involved than, for example, standard compressed sensing.\\

{\noindent\textbf{Theorem~\ref{thm:nmf} (informal). }}\emph{Consider a rank-$r$ matrix $M\in\mathbb{R}^{n \times m}$, where $M = WH$, $W\in\mathbb{R}^{n\times r}$ and $H \in \mathbb{R}^{r \times m}$. Let the columns of $W$ be sparse with the non-zero entries chosen at random.  Given the compressed measurements $\tilde{M}=PM$ for a measurement matrix $P\in\mathbb{R}^{d\times n}$, under suitable conditions on $P,n,m,d$ and the sparsity, $\tilde{M}=(PW)H$ is the} sparsest \emph{rank-$r$ factorization of $\tilde{M}$ with high probability, in which case performing sparse recovery on the columns of $(PW)$ will yield the true factors $W$.}\\

While Theorem~\ref{thm:nmf} provides guarantees on the quality of the sparsest rank-$r$ factorization, it does not directly address the algorithmic question of how to find such a factorization efficiently. For some of the settings of interest, such as sparse PCA, efficient algorithms for recovering this sparsest factorization are known, under some mild assumptions on the data \citep{amini2009learning,zhou2011godec,deshpande2014sparse, papailiopoulos2013sparse}.
In such settings, Theorem~\ref{thm:nmf} guarantees that we can efficiently recover the correct factorization. 

For other matrix factorization problems such as NMF, the current algorithmic understanding of how to recover the factorization is incomplete even for uncompressed data, and guarantees for provable recovery require strong assumptions such as separability \citep{arora2012learning}. As the original problem (computing NMF of the uncompressed matrix $M$) is itself NP-hard \citep{vavasis2009complexity}, hence one should not expect an analog of Theorem~\ref{thm:nmf} to avoid solving a computationally hard problem and guarantee efficient recovery in general.
In practice, however, NMF algorithms are commonly observed to yield sparse factorizations on real-world data~\citep{lee1999learning,hoyer2004non} and there is substantial work on explicitly inducing sparsity via regularized NMF variants~\citep{hoyer2004non,li2001learning,kim2008sparse,peharz2012sparse}. 
%Despite this, we do have algorithms that perform well in practice on many real-world applications of NMF:
%for example, NMF has successfully applied to sparse dictionary learning problems~\citep{lee1999learning,hoyer2004non}.
%and there is significant work on explicitly inducing sparsity in regularized variants of NMF~\citep{hoyer2004non,li2001learning,kim2008sparse,peharz2012sparse}. 
In light of this empirically demonstrated ability to compute sparse NMF, Theorem~\ref{thm:nmf} provides theoretical grounding for why \textsc{Factorize-Recover} should yield accurate reconstructions of the original factors. 

Our theoretical results assume a noiseless setting, but real-world data is usually noisy and only approximately sparse.
Thus, we demonstrate the practical applicability of \textsc{Factorize-Recover} through experiments on both synthetic benchmarks as well as several real-world gene expression datasets. 
We find that performing NMF on compressed data achieves reconstruction accuracy comparable to or better than factorizing the recovered (uncompressed) data at a fraction of the computation time.

In addition to our results on matrix factorization, we show the following analog to Theorem~\ref{thm:nmf} for compressed CP tensor decomposition. The proof in this case follows in a relatively straightforward fashion from the techniques developed for our matrix factorization result.\\

\noindent \textbf{Proposition~\ref{prop:td} (informal).} 
\emph{Consider a rank-$r$ tensor $T\in \mathbb{R}^{n\times m_1 \times m_2}$ with factorization $T=\sum_{i=1}^{r} A_i \otimes B_i \otimes C_i$, where $A$ is sparse with the non-zero entries chosen at random. Under suitable conditions on $P$, the dimensions of the tensor, the projection dimension and the sparsity, $\tilde{T}=\sum_{i=1}^{r}(PA_i) \otimes B_i \otimes C_i$ is the unique factorization of the compressed tensor $\tilde{T}$ with high probability, in which case performing sparse recovery on the columns of $(PA)$ will yield the true factors $A$.}\\

As in the case of sparse PCA, there is an efficient algorithm for finding this unique tensor factorization, as tensor decomposition can be computed efficiently when the factors are linearly independent (see e.g. \citet{kolda2009tensor}). 
We empirically validate our approach for tensor decomposition on a real-world EEG dataset, demonstrating that factorizations from compressed measurements can yield interpretable factors that are indicative of the onset of seizures.

\section{Related Work}

There is an enormous body of algorithmic work on computing matrix and tensor decompositions more efficiently using random projections, usually by speeding up the linear algebraic routines that arise in the computation of these factorizations.
This includes work on randomized SVD~\citep{halko2011finding,clarkson2013low}, NMF \citep{wang2010efficient,tepper2016compressed} and CP tensor decomposition \citep{battaglino2017practical}.  
This work is rather different in spirit, as it leverages projections to accelerate certain components of the algorithms, but still requires repeated accesses to the original uncompressed data.
In contrast, our methods apply in the setting where we are only given access to the compressed data.

As mentioned in the introduction, learning from compressed data has been widely studied, yielding strong results for many  learning tasks such as linear classification \citep{calderbank2009compressed,durrant2010compressed}, multi-label prediction \citep{hsu2009multi}  and regression \citep{zhou2009compressed,maillard2009compressed}. 
In most of these settings, the goal is to obtain a good predictive model in the compressed space itself, instead of recovering the model in the original space. 
A notable exception to this is previous work on performing PCA and matrix co-factorization on compressed data \citep{fowler2009compressive,ha2015robust,yoo2011matrix}; we extend this line of work by considering sparse matrix decompositions like sparse PCA and NMF.
To the best of our knowledge, ours is the first work to establish conditions under which sparse matrix factorizations can be recovered directly from compressed data.

Compressive sensing techniques have been extended to reconstruct higher-order signals from compressed data.
For example, Kronecker compressed sensing \citep{duarte2012kronecker} can be used to recover a tensor decomposition model known as Tucker decomposition from compressed data \citep{caiafa2013multidimensional,caiafa2015stable}. Uniqueness results for reconstructing the tensor are also known in certain regimes \citep{sidiropoulos2012multi}. 
Our work extends the class of models and measurement matrices for which uniqueness results are known and additionally provides algorithmic guarantees for efficient recovery under these conditions.

% There has also been some work on learning from compressed data for unsupervised tasks such as PCA \citep{fowler2009compressive,ha2015robust}, where the goal is to learn the factors in the original space.

From a technical perspective, the most relevant work is \citet{spielman2012exact}, which considers the sparse coding problem.  
Although their setting differs from ours, the technical cores of both analyses involve characterizing the sparsity patterns of linear combinations of random sparse vectors.

\section{Compressed Factorization}

In this section, we first establish preliminaries on compressive sensing, followed by a description of the measurement matrices used to compress the input data.
Then, we specify the algorithms for compressed matrix and tensor factorization that we study in the remainder of the paper.\\

% Our approach consists of three simple steps. In the first step, we project down the original high dimensional optimization problem to low dimensions. In the second step, we solve the low dimensional optimization problem. In the third step, we use compressive sensing techniques to recover back the original solution to the high dimensional optimization problem from the solution to the low dimensional optimization problem. 
\noindent\textbf{Notation.}\hspace{0.25em} Let $[n]$ denote the set $\{1,2,\dots,n\}$. For any matrix $A$, we denote its $i$th column as $A_i$. %Throughout, we denote the projection matrix by $P$.  
%It will be useful for us to define notation for running sparse recovery via a LP on a compressed vector $w\in \mathbb{R}^d$ to obtain a vector in the original $n$ dimensional space. 
For a matrix $P\in\mathbb{R}^{d\times n}$ such that $d < n$, define:
%\vspace{-3pt}
\begin{equation}
\begin{aligned}
& \mathcal{R}_P(w)= \underset{x : Px=w}{\text{argmin}}
 \;\|x\|_1
%& \text{subject to}
%& & Px = w
\end{aligned}\label{eq:basis-pursuit}
\end{equation}
%\vspace{-3pt}
as the sparse recovery operator on $w \in \mathbb{R}^n$. We omit the subscript $P$ when it is clear from context.\\

\noindent\textbf{Background on Compressive Sensing.}\hspace{0.25em}
% LP formulation. Sparse matrices. SSMP, lower reconstruction accuracy.
%Since we will repeatedly leverage a sparse recovery subroutine, it will be helpful to briefly summarize the main ideas in this domain. 
In the compressive sensing framework, there is a sparse signal $x\in \mathbb{R}^n$ for which we are given $d\ll n$ linear measurements $Px$, where $ P \in \mathbb{R}^{d\times n}$ is a known measurement matrix. 
The goal is to recover $x$ using the measurements $Px$, given the prior knowledge that $x$ is sparse. 
Seminal results in compressive sensing \cite{donoho2006compressed,candes2006near,candes2008restricted} show that if the original solution is $k$-sparse, then it can be exactly recovered from $d=O(k\log n)$ measurements by solving a linear program (LP) of the form (\ref{eq:basis-pursuit}). 
More efficient recovery algorithms than the LP for solving the problem  are also known~\citep{berinde2008practical,indyk2008near,berinde2009sequential}. 
However, these algorithms typically require more measurements in the compressed domain to achieve the same reconstruction accuracy as the LP formulation \citep{berinde2009sequential}.\\
%We therefore use the LP formulation for sparse recovery 
%We also observed a significant gap in our evaluation, and hence use the LP formulation for the rest of the paper.

\noindent\textbf{Measurement Matrices.}\hspace{0.25em}
In this work, we  consider sparse, binary measurement (or projection) matrices $P\in\{0,1\}^{d \times n}$ where each column of $P$ has  $p$ non-zero entries chosen uniformly and independently at random.
For our theoretical results, we set $p=O(\log n)$.
Although the first results on compressive sensing only held for dense matrices \cite{donoho2006compressed,candes2008restricted,candes2006near}, subsequent work has shown that sparse, binary matrices can also be used for compressive sensing \cite{berinde2008combining}. %(we discuss this further in Section \ref{sec:nmf_theory}). 
In particular, Theorem 3 of \citet{berinde2008combining} shows that the recovery procedure in \eqref{eq:basis-pursuit} succeeds with high probability for the class of $P$ we consider if the original signal is $k$-sparse and $d=\Omega(k\log n)$. In practice, sparse binary projection matrices can arise due to physical limitations in sensor design (\emph{e.g.}, where measurements are sparse and can only be performed additively) or in applications of non-adaptive group testing~\citep{indyk2010efficiently}. \\

%\vspace{-8pt}
\noindent\textbf{Low-Rank Matrix Factorization.}\hspace{0.25em} We assume that each sample is an $n$-dimensional column vector in uncompressed form. 
Hence, the uncompressed matrix $M\in \mathbb{R}^{n\times m}$ has $m$ columns corresponding to $m$ samples, and we assume that it has some rank-$r$ factorization: $M=WH$, where ${W} \in \mathbb{R}^{n \times r}$, ${H} \in \mathbb{R}^{r \times m}$, and the columns of $W$ are $k$-sparse. 
We are given the compressed matrix $\tilde{M}=PM$ corresponding to the $d$-dimensional projection $Pv$ for each sample $v\in \mathbb{R}^n$. 
We then compute a low-rank factorization using the following algorithm:

\vspace{0.5em}
\begin{algorithm}
	%\setstretch{0.9}
	\DontPrintSemicolon
	\SetAlgoLined
	%\SetAlgoNoEnd
	\SetKwFunction{FDot}{Dot}
	\SetKwFunction{FUpdate}{Update}

	\newcommand\mycommfont[1]{\rmfamily{#1}}
	\SetCommentSty{mycommfont}
	\SetKwComment{Comment}{// }{}
	\SetKwComment{LeftComment}{\Statex // }{}
	\Input{Compressed matrix $\tilde{M}=PM$, projection matrix $P$}{}
    \Alg{Outputs estimates $(\hat{W},\hat{H})$ of $(W,H)$}
    {
    Compute rank-$r$ factorization of $\tilde{M}$ to obtain $\tilde{W},\tilde{H}$\;
    
    Set $\hat{H}\gets \tilde{H}$\;
    
    \For{$1\le i\le r$}
    {
    \Comment{Solve (\ref{eq:basis-pursuit}) to recover $\hat{W}_i$ from $\tilde{W}_i$}
    Set $\hat{W}_i\gets\mathcal{R}(\tilde{W}_i)$\;
    }
    }
	\caption{Compressed Matrix Factorization}
	\label{alg:nmf}
\end{algorithm}
\vspace{0.5em}

% old description 
% \emph{Non-Negative Matrix Factorization.} We assume that in uncompressed form, each sample is an $n$-dimensional column vector. Hence, the uncompressed matrix $M\in \mathbb{R}^{n\times m}$ has $m$ columns corresponding to $m$ samples, and we assume that it has some rank-$r$ non-negative factorization $M=WH$ where the columns of $W$ are $k$-sparse. For each sample $v$, we are given the compressed $d$-dimensional projection $Pv$ corresponding to a projection matrix $P \in \mathbb{R}^{d \times n}$; this yields the compressed matrix $\tilde{M}=PM$. We then perform rank-$r$ NMF on $\tilde{M}$ to obtain a factorization $\tilde{M}=\tilde{W}\tilde{H}$. We set our estimate of $H$ to be $\tilde{H}$. 
%To recover $W$ from $\tilde{W}$, we solve the sparse recovery problem (Eq.~\ref{eq:basis-pursuit}) for each column $W_i$ of $W$, setting the estimate for the $i$th column of $W_i$ as the recovered vector $\mathcal{R}(\tilde{W}_i)$. To obtain all $r$ columns of $W$, we solve $r$ instances of Eq. \ref{eq:basis-pursuit}.\footnote{In order to enforce non-negativity of $W$, we set any negative entries of the recovered solution to 0.}

%\begin{enumerate} 

%\vspace{-8pt}
	 \noindent\textbf{CP Tensor Decomposition.}\hspace{0.25em}
	 As above, we assume that each sample is $n$-dimensional and $k$-sparse. The samples are now indexed by two coordinates $y\in [m_1]$ and $z\in [m_2]$, and hence can be represented by a tensor $T\in\mathbb{R}^{n\times m_1 \times m_2 }$.
%The coordinates $y$ and $z$ could correspond to different channels of measurement, different trials of the experiment, or some temporal aspect. 
We assume that $T$ has some rank-$r$  factorization $T=\sum_{i=1}^r A_i \otimes B_i \otimes C_i$, where the columns of $A$ are $k$-sparse. 
Here $\otimes$ denotes the outer product: if $a\in \mathbb{R}^n, b \in \mathbb{R}^{m_1}, c \in \mathbb{R}^{m_2}$ then $a\otimes b \otimes c \in \mathbb{R}^{n\times m_1 \times m_2}$ and $(a\otimes b\otimes b)_{ijk}=a_ib_jc_k$. 
This model, CP decomposition, is the most commonly used model of tensor decomposition.  
For a measurement matrix  $P \in \mathbb{R}^{d \times n}$, we are given a projected tensor $\tilde{T}\in \mathbb{R}^{d \times m_1 \times m_2}$ corresponding to a $d$ dimensional projection $Pv$ for each sample $v$. Algorithm \ref{alg:td} computes a low-rank factorization of $T$ from $\tilde{T}$.

\vspace{0.5em}
\begin{algorithm}
    %\setstretch{0.9}
	\DontPrintSemicolon
	\SetAlgoLined
	%\SetAlgoNoEnd
	\SetKwFunction{FDot}{Dot}
	\SetKwFunction{FUpdate}{Update}
   
	\newcommand\mycommfont[1]{\rmfamily{#1}}
	\SetCommentSty{mycommfont}
	\SetKwComment{Comment}{// }{}
	\SetKwComment{LeftComment}{\Statex // }{}
	\Input{Compressed tensor $\tilde{T}$, projection matrix $P$ }{}
    \Alg{Outputs estimates $(\hat{A},\hat{B},\hat{C})$ of $(A,B,C)$}{
    Compute rank-$r$ TD of $\tilde{T}$: $\tilde{T}=\sum_{i=1}^r \tilde{A}_i \otimes \tilde{B}_i \otimes \tilde{C}_i$\;
    
    Set $\hat{B}\gets \tilde{B}$, $\hat{C}\gets \tilde{C}$\;
    
    \For{$1\le i\le r$}
    {
    \Comment{Solve (\ref{eq:basis-pursuit}) to recover $\hat{A}_i$ from $\tilde{A}_i$}
    Set $\hat{A}_i\gets\mathcal{R}(\tilde{A}_i)$\;
    }
    
    }
	\caption{Compressed CP Tensor Decomposition}
	\label{alg:td}
\end{algorithm}  
\vspace{0.5em}
     
% old description 
%	 \emph{CP Tensor Decomposition.}
%	 Similar to NMF, here we also assume that each sample is $n$-dimensional and $k$-sparse, but the samples are now indexed by two coordinates $y\in [m_1]$ and $z\in [m_2]$, and hence can be represented by a tensor $T\in\mathbb{R}^{n\times m_1 \times m_2 }$. The coordinates $y$ and $z$ could correspond to different channels of measurement, different trials of the experiment, or some temporal aspect. We assume that $T$ has some rank-$r$  factorization $T=\sum_{i=1}^r A_i \otimes B_i \otimes C_i$, where the columns of $A$ are $k$-sparse; our goal is to recover this factorization. For a projection matrix  $P \in \mathbb{R}^{d \times n}$, we are given a projected tensor $\tilde{T}\in \mathbb{R}^{d \times m_1 \times m_2}$ corresponding to a $d$ dimensional projection $Pv$ for each sample $v$, where the samples are indexed by the same coordinates $y$ and $z$ as before. We perform CP decomposition on $\tilde{T}$ to obtain a factorization $\tilde{T}=\sum_{i=1}^r \tilde{A}_i \otimes \tilde{B}_i \otimes \tilde{C}_i$. We set our estimate of $B$ to be $\tilde{B}$ and $C$ to be $\tilde{C}$. To recover $A$ we solve $r$ instances of the sparse recovery problem (Eq. ~\ref{eq:basis-pursuit}), setting the estimate for the $i$th column of $A$ as $\mathcal{R}(\tilde{A}_i)$.

We now describe our formal results for matrix and tensor factorization.

\section{Theoretical Guarantees}\label{sec:theory}

In this section, we establish conditions under which \textsc{Factorize-Recover} will provably succeed for matrix and tensor decomposition on compressed data.

%. This establishes the conditions under which the factors in the projected space correspond to a projection of the factors in the original space, a prerequisite for being able to use compressive sensing to recover the factors in the original space.  %We believe it is possible to show similar guarantees for linear classification, using results on the quality of 

%To summarize these results, sparse, binary projection matrices preserve the uniqueness of these factorization for column sparsity $p=O(\log n)$ and projection dimension $d=\Omega((r+k)\log n)$. Note that the bounds in our theoretical results are tight up to logarithmic factors, as projecting the matrix or the tensor to dimension below the rank will not preserve uniqueness of the decomposition, and $O(k\log n)$ measurements are necessary to successfully recover $k$-sparse signals using compressed sensing.

\subsection{Sparse Matrix Factorization}
\label{sec:nmf_theory}

\vspace{-4pt}
The main idea is to show that with high probability, $\tilde{M}=(PW)H$ is the \emph{sparsest} factorization of $\tilde{M}$ in the following sense: for any other factorization $\tilde{M}={W}'H'$, $W'$ has strictly more non-zero entries than $(PW)$.
%For what conditions on the sparsity and projection dimension do sparse, binary matrices preserve the uniqueness of sparse factorizations after random projections? 
%We show that sparse, binary matrices $P\in \{0,1\}^{d\times n}$ with each column of $P$ containing $p=O(\log n)$ non-zero entries chosen uniformly and independently at random preserve uniqueness with high probability as long as the projection dimension $d$ satisfies $d=\Omega((r+k)\log n)$.
It follows that the factorization $(PW)H$ is the optimal solution for a sparse matrix factorization of $\tilde{M}$ that penalizes non-zero entries of $\tilde{W}$.
To show this uniqueness property, we show that the projection matrices satisfy certain structural conditions with high probability, namely that they correspond to adjacency matrices of \emph{bipartite expander} graphs \cite{hoory2006expander}, which we define shortly. We first formally state our theorem:

\begin{thm}\label{thm:nmf}
	Consider a rank-$r$ matrix $M\in \mathbb{R}^{n\times m}$ which has factorization $M=WH,$ for $H\in \mathbb{R}^{r\times m}$ and $W\in \mathbb{R}^{n\times r}$. Assume $H$ has full row rank and $W=B \odot Y$, where $B\in \{0,1\}^{n\times r}$, $Y\in \mathbb{R}^{n\times r}$ and $\odot$ denotes the elementwise product. Let each column of $B$ have $k$ non-zero entries chosen uniformly and independently at random, and each entry of $Y$ be an independent random variable drawn from any continuous distribution.\footnote{For example, a Gaussian distribution, or absolute value of Gaussian in the NMF setting.} Assume $k>C$, where $C$ is a fixed constant. Consider the projection matrix $P\in \{0,1\}^{d\times n}$ where each column of $P$ has $p=O(\log n)$ non-zero entries chosen independently and uniformly at random. Assume $d=\Omega((r+k)\log n)$. Let $\tilde{M}=PM$. Note that $\tilde{M}$ has one possible factorization $\tilde{M}=\tilde{W}H$ where $\tilde{W}=PW$. For some fixed $\beta>0$, with failure probability at most $(r/n)e^{-\beta k} +(1/n^5)$, $\tilde{M}=\tilde{W}H$ is the sparsest possible factorization in terms of the left factors: for any other rank-$r$ factorization $\tilde{M}={W'}{H'}$, $\lpnorm{\tilde{W}}{0}<\lpnorm{{W'}}{0}$.
%	\begin{enumerate}[(a)]
%		\vspace{-8pt}
%		\item For any other rank $r$ factorization $M={H'}{W'}$, $\lpnorm{W}{0}<\lpnorm{{W'}}{0}$ with failure probability at most $re^{-\beta k}/n$ for some $\beta>0$.
%		\vspace{-10pt}
%		\item For any other rank $r$ factorization $\tilde{M}={H'}{W'}$, $\lpnorm{\tilde{W}}{0}<\lpnorm{{W'}}{0}$ with failure probability at most $re^{-\beta k}/n$ for some $\beta>0$.
%	\end{enumerate}  
\end{thm}

%The main contribution of Theorem \ref{thm:nmf} is Part (b), which shows uniqueness results in the projected space. This is significantly more challenging than showing uniqueness in the original space, because in the original space we could use the fact the entries of $W$ are drawn independently at random to  simplify our arguments. However we lose access to this independence after the random projection step as the entries of the projected matrix are no longer independent. Instead, the proof in the projected space uses the expansion properties of the projection matrix $P$ to control the dependencies, enabling us to prove uniqueness results in the projected space. 

%Theorem \ref{thm:nmf} only assumes that $H$ is full row rank, and 
\vspace{-4pt}
Theorem~\ref{thm:nmf} shows that if the columns of $W$ are $k$-sparse, then projecting into $\Omega((r+k)\log n)$ dimensions preserves uniqueness, with failure probability at most $(r/n)e^{-\beta k}+(1/n)^5$, for some constant $\beta>0$. As real-world matrices have been empirically observed to be typically close to low rank, the $(r/n)$ term is usually small for practical applications. Note that the requirement for the projection dimension being at least $\Omega((r+k)\log n)$ is close to optimal, as even being able to uniquely recover a $k$-sparse $n$-dimensional vector $x$ from its projection $Px$ requires the projection dimension to be at least $\Omega(k\log n)$; we also cannot hope for uniqueness for projections to dimensions below the rank $r$. We also remark that the distributional assumptions on $P$ and $W$ are quite mild, as any continuous distribution suffices for the non-zero entries of $W$, and the condition on the set of non-zero coordinates for $P$ and $W$ being chosen uniformly and independently for each column can be replaced by a deterministic condition that $P$ and $W$ are adjacency matrices of bipartite expander graphs.
%We emphasize that the assumption on the nonzero entries of $W$ being Gaussian is stated for simplicity and any other continuous distribution for the non-zero entries will also suffice. 
We provide a proof sketch below, with the full proof deferred to the Appendix.\\%~\ref{sec:nmf_app}.

%\vspace{-12pt}
%\begin{proofsketch}
\noindent {\emph{Proof sketch. }} We first show a simple Lemma that for any other factorization $\tilde{M}=W'H'$, the column space of $W'$ and $\tilde{W}$ must be the same (Lemma 5 %\ref{lem:sp_eq} 
 in the Appendix). Using this, for any other factorization $\tilde{M}=W'H'$, the columns of $W'$ must lie in the column space of $\tilde{W}$, and hence  our goal will be to prove that the columns of $\tilde{W}$ are the sparsest vectors in the column space of $\tilde{W}$, which implies that for any other  factorization $\tilde{M}={W'}{H'}$, $\lpnorm{\tilde{W}}{0}<\lpnorm{{W'}}{0}$. 

The outline of the proof is as follows. It is helpful to think of the matrix $\tilde{W}\in \mathbb{R}^{d \times r}$ as corresponding to the adjacency matrix of an unweighted bipartite graph $G$ with $r$ nodes on the left part $U_1$ and $d$ nodes on the right part $U_2$, and an edge from a node $u\in U_1$ to a node $v\in U_2$ if the corresponding entry of $\tilde{W}$ is non-zero. For any subset $S$ of the columns of $\tilde{W}$, define $N(S)$ to be the subset of the rows of $\tilde{W}$ which have a non-zero entry in at least one of the columns in $S$. In the graph representation $G$, $N(S)$ is simply the neighborhood of a subset $S$ of vertices in the left part $U_1$. In part (a) we argue that the if we take any subset $S$ of the columns of $\tilde{W}$, $|N(S)|$ will be large. This implies that taking a linear combination of all the $S$ columns will result in a vector with a large number of non-zero entries---unless the non-zero entries cancel in many of the columns. In part (b), by using the properties of the projection matrix $P$ and the fact that the non-zero entries of the original matrix $W$ are drawn from a continuous distribution, we show this happens with zero probability.

The property of the projection matrix that is key to our proof is that it is the adjacency matrix of a bipartite expander graph, defined below.

\begin{definition}
	Consider a bipartite graph $R$ with $n$ nodes on the left part and $d$ nodes on the right part such that every node in the left part has degree $p$. We call $R$ a $(\gamma n,\alpha)$ expander if every subset of at most $t\le \gamma n$ nodes in the left part has at least $\alpha tp$ neighbors in the right part.
\end{definition}
\vspace{-4pt}
 It is well-known that adjacency matrices of random bipartite graphs have good expansion properties under suitable conditions \citep{vadhan2012pseudorandomness}. For completeness, we show in Lemma 6 %\ref{expander} 
 in the Appendix that a randomly chosen matrix $P$ with $p$ non-zero entries per column is the adjacency matrix of a $(\gamma n, 4/5)$ expander for $\gamma n=d/(pe^5)$ with failure probability $(1/n^5)$, if $p=O(\log n)$. %For completeness, we provide a proof in Lemma 6 %\ref{expander} in the Appendix. 
 Note that part (a) is a requirement on the graph $G$ for the matrix $\tilde{W}$ being a bipartite expander. In order to show that $G$ is a bipartite expander, we show that  with high probability $P$ is a bipartite expander, and the matrix $B$ corresponding to the non-zero entries of $W$ is also a bipartite expander. $G$ is a cascade of these bipartite expanders, and hence is also a bipartite expander.  %This shows that for any subset $S$ of columns of $W$, the number of rows of $W$ which have non-zero entries in at least one of the columns of $S$ is large. By combining this with the expander property of $P$, it follows that $|N(S)|$ is large for any subset $S$ of the columns of $\tilde{W}$.
 
 For part (b), we need to deal with the fact that the entries of $\tilde{W}$ are no longer independent because the projection step leads to each entry of $\tilde{W}$ being the sum of multiple entries of $W$. However, the structure of $P$ lets us control the dependencies, as each entry of $W$ appears at most $p$ times in $\tilde{W}$. Note that for a linear combination of any subset of $S$ columns, $|N(S)|$ rows have non-zero entries in at least one of the $S$ columns, and $|N(S)|$ is large by part (a). Since each entry of $W$ appears at most $p$ times in $\tilde{W}$, we can show that with high probability at most $|S|p$ out of the $|N(S)|$ rows with non-zero entries are zeroed out in any linear combination of the $S$ columns. Therefore, if $|N(S)|-|S|p$ is large enough, then any linear combination of $S$ columns has a large number of non-zero entries and is not sparse. This implies that the columns of $\tilde{W}$ are the sparsest columns in its column space. \hfill $\blacksquare$\\% We defer the full proof to the Appendix.
%{\flushright $\blacksquare$}

%\vspace{-4pt}

A natural direction of future work is to relax some of the assumptions of Theorem \ref{thm:nmf}, such as requiring independence between the entries of the $B$ and $Y$ matrices, and among the entries of the matrices themselves. It would also be interesting to show a similar uniqueness result under weaker, deterministic conditions on the left factor matrix $W$ and the projection matrix $P$. Our result is a step in this direction and shows uniqueness if the non-zero entries of $W$ and $P$ are adjacency matrices of bipartite expanders, but it would be interesting to prove this under more relaxed assumptions.

%Also, our result shows uniqueness of the factorization if the non-zero entries of the left factors $W$ and the matrix $P$ are bipartite expanders, but it would be interesting to prove a similar uniqueness result under weaker deterministic conditions on the matrices.

%It would also be interesting to show a similar uniqueness result under weaker deterministic conditions on the left factor matrix $W$ and the matrix $P$. Our result shows uniqueness if the non-zero entries of the left factors $W$ and the matrix $P$ are bipartite expanders, but it would be interesting to prove this under weaker deterministic conditions.

%\end{proofsketch}

%\vspace{-10pt}
\subsection{Tensor Decomposition}
\label{sec:tensor-theory}

It is easy to show uniqueness for tensor decomposition after random projection since tensor decomposition is unique under mild conditions on the factors~\citep{kruskalthree,kolda2009tensor}. 
%For example, tensors have a unique decomposition if all the factor matrices are full rank. 
Formally: %and defer the proof and further details to Appendix~\ref{sec:tensor_app}.

%Our uniqueness guarantee for tensor decomposition is stated under analogous conditions as uniqueness for sparse factorizations, but we can additionally show that the projected factorization is not just the sparsest factorization of the projected tensor---but the only possible factorization. We state the informal guarantee here, the formal statement can be found in the Appendix. %, which is true with high probability over the randomness in choosing $P$.

\begin{prop}\label{prop:td}
	Consider a rank-$r$ tensor $T\in \mathbb{R}^{n\times m_1 \times m_2}$ which has factorization $T=\sum_{i=1}^{r} A_i \otimes B_i \otimes C_i$, for $A\in \mathbb{R}^{n\times r}$, $B\in \mathbb{R}^{m_1\times r}$ and $C\in \mathbb{R}^{m_1\times r}$. Assume $B$ and $C$ have full column rank and $A=X \odot Y$, where each column of $X$ has exactly $k$ non-zero entries chosen uniformly and independently at random, and each entry of $Y$ is an independent random variable drawn from any continuous distribution. Assume $k>C$, where $C$ is a fixed constant. Consider a projection matrix $P\in \{0,1\}^{d\times n}$ with $d=\Omega((r+k)\log n)$ where each column of $P$ has exactly $p=O(\log n)$ non-zero entries chosen independently and uniformly at random. Let $\tilde{T}$ be the projection of $T$ obtained by projecting the first dimension. Note that $\tilde{T}$ has one possible factorization $\tilde{T}=\sum_{i=1}^{r}(PA_i) \otimes B_i \otimes C_i$. For a fixed $\beta>0$, with failure probability at most $(r/n)e^{-\beta k} +(1/n^5)$, $PA$ has full column rank, and hence this is a unique factorization of $\tilde{T}$.
	%	\begin{enumerate}[(a)]
	%		\vspace{-8pt}
	%		\item For any other rank $r$ factorization $M={H'}{W'}$, $\lpnorm{W}{0}<\lpnorm{{W'}}{0}$ with failure probability at most $re^{-\beta k}/n$ for some $\beta>0$.
	%		\vspace{-10pt}
	%		\item For any other rank $r$ factorization $\tilde{M}={H'}{W'}$, $\lpnorm{\tilde{W}}{0}<\lpnorm{{W'}}{0}$ with failure probability at most $re^{-\beta k}/n$ for some $\beta>0$.
	%	\end{enumerate}  
\end{prop}

%The proof directly follows from our proof of Theorem \ref{thm:td} and uniqueness properties of tensor decompositions; we include it in the Appendix for completeness. 
%\vspace{-4pt}
Note that efficient algorithms are known for recovering tensors with linearly independent factors \citep{kolda2009tensor} and hence under the conditions of Proposition \ref{prop:td} we can efficiently find the factorization in the compressed domain from which the original factors can be recovered. In the Appendix, we also show that we can provably recover factorizations in the compressed space using variants of the popular alternating least squares algorithm for tensor decomposition, though these algorithms require stronger assumptions on the tensor such as incoherence. %These guarantees are stated 
%in Section \ref{sec:tensor_app}  in the Appendix.

The proof of Proposition \ref{prop:td} is direct given the results established in Theorem \ref{thm:nmf}. We use the fact that tensors have a unique decomposition whenever the underlying factors $(PA),B,C$ are full column rank \citep{kruskalthree}. By our assumption, $B$ and $C$ are given to be full rank. The key step is that by the proof of Theorem \ref{thm:nmf}, the columns of $PA$ are the sparsest columns in their column space. Therefore, they must be linearly independent, as otherwise the all zero vector will lie in their column space. Therefore, $(PA)$ has full column rank, and Proposition \ref{prop:td} follows.

\section{Experiments}
\label{sec:experiments}

We support our theoretical uniqueness results with experiments on real and synthetic data. 
On synthetically generated matrices where the ground-truth factorizations are known, we show that standard algorithms for computing sparse PCA and NMF converge to the desired solutions in the compressed space~(\S\ref{sec:experiments-synthetic}). 
We then demonstrate the practical applicability of compressed factorization with experiments on gene expression data~(\S\ref{sec:experiments-nmf}) and EEG time series~(\S\ref{sec:experiments-td}).

\subsection{Synthetic Data}
\label{sec:experiments-synthetic}

\begin{figure*}[]
\includegraphics[width=0.463\textwidth]{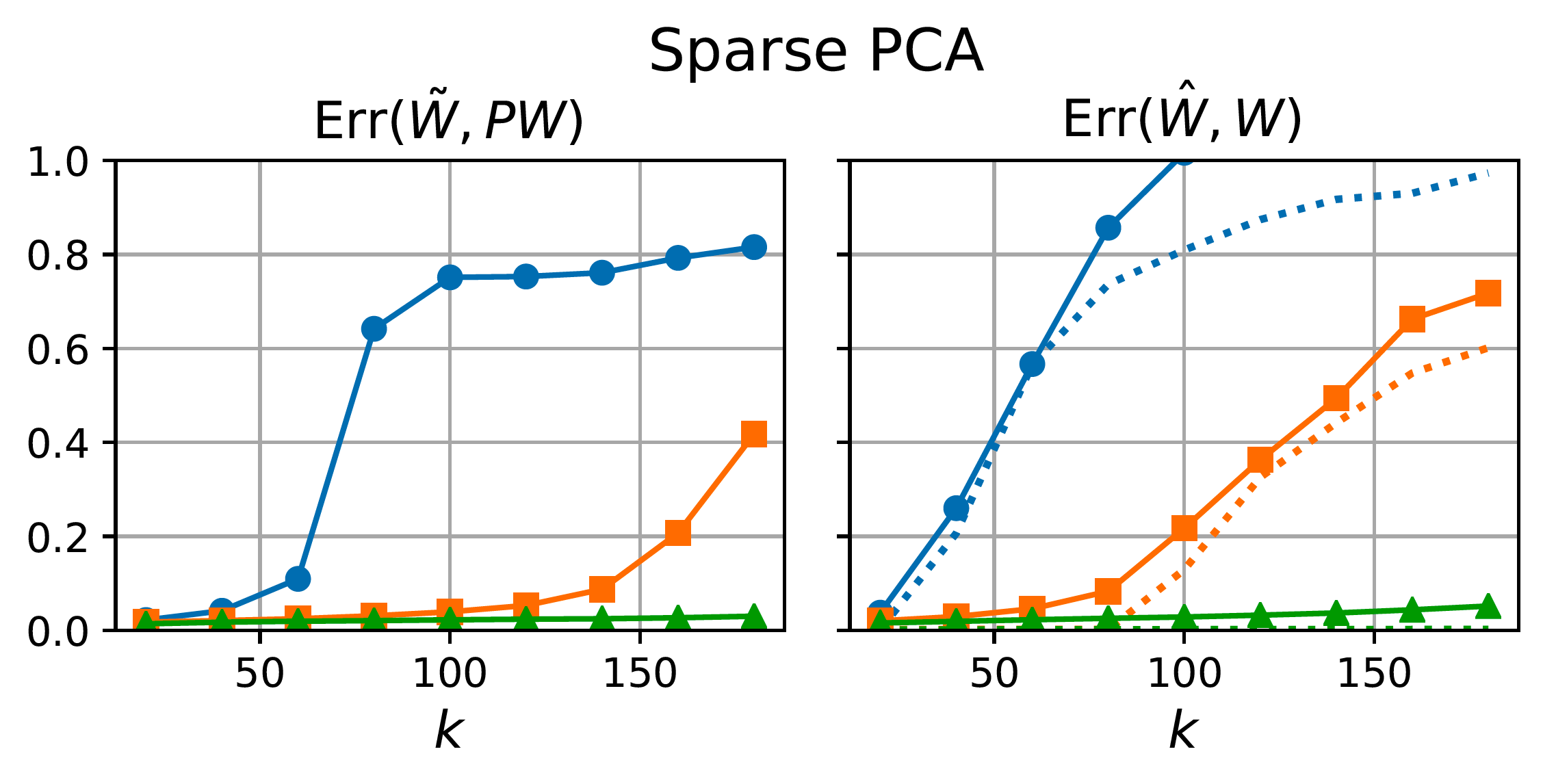}
\includegraphics[width=0.537\textwidth]{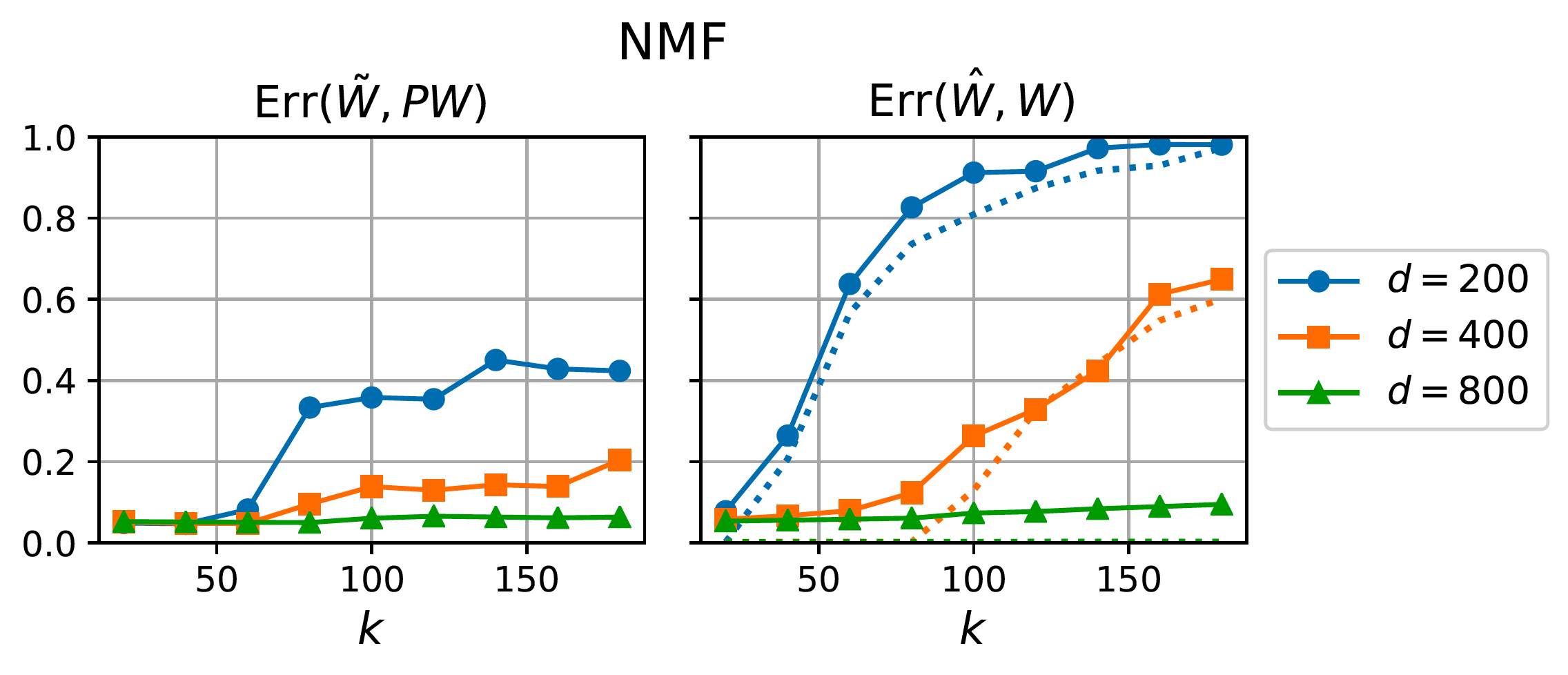}
\caption{Approximation errors $\mathrm{Err}(X, X_*) \coloneqq \|X - X_*\|_F / \|X_*\|_F$ for sparse PCA and NMF on synthetic data with varying column sparsity $k$ of $W$ and projection dimension $d$.
The values of $d$ correspond to $10\times$, $5\times$, and $2.5\times$ compression respectively.
$\mathrm{Err}(\tilde{W}, PW)$ measures the distance between factors in the compressed domain: low error here is necessary for accurate sparse recovery.
$\mathrm{Err}(\hat{W}, W)$ measures the error after sparse recovery: the recovered factors $\hat{W}$ typically incur only slightly higher error than the oracle lower bound (dotted lines) where $PW$ is known exactly.
%Sparse PCA and NMF successfully recover the desired factorization $\tilde{M}=(PW)H$ when $d$ is sufficiently large relative to the sparsity. 
%The right panels show approximation errors of $W$ relative to the recovered $\hat{W}$. 
%The recovered factors $\hat{W}$ are typically only slightly worse in approximation error than the oracle lower bound where $PW$ is given (dotted lines).}
}
\label{fig:synthetic}
\end{figure*}

%\vspace{-4pt}
We provide empirical evidence that standard algorithms for sparse PCA and NMF converge in practice to the desired sparse factorization $\tilde{M} = (PW)H$---in order to achieve accurate sparse recovery in the subsequent step, it is necessary that the compressed factor $\tilde{W}$ be a good approximation of $PW$.
For sparse PCA, we use alternating minimization with LARS \citep{zou2006sparse}, and for NMF, we use projected gradient descent \citep{lin2007projected}.\footnote{For sparse PCA, we report results for the setting of the $\ell_1$ regularization parameter that yielded the lowest approximation error. We did not use an $\ell_1$ penalty for NMF. We give additional details in the Appendix.}
Additionally, we evaluate the quality of the factors obtained after sparse recovery by measuring the approximation error of the recovered factors $\hat{W}$ relative to the true factors $W$.

We generate synthetic data following the conditions of Theorem 1. For sparse PCA, we sample matrices $W = B \odot Y$ and $H$, where each column of $B \in \{0,1\}^{n\times r}$ has $k$ non-zero entries chosen uniformly at random, $Y_{ij}\overset{\text{iid}}{\sim} N(0, 1)$, and $H_{ij}\overset{\text{iid}}{\sim} N(0, 1)$. For NMF, an elementwise absolute value function is applied to the values sampled from this distribution. 
The noisy data matrix is $M = WH + \mathcal{E}$, where the noise term $\mathcal{E}$ is a dense random Gaussian matrix scaled such that $\|\mathcal{E}\|_F / \|WH\|_F = 0.1$. We observe $\tilde{M} = PM$, where $P$ has $p=5$ non-zero entries per column (in the Appendix, %Appendix \ref{app:projection-sparsity}
we study the effect of varying $p$ on the error).

Figure~\ref{fig:synthetic} shows our results on synthetic data with $m=2000$, $n=2000$, and $r=10$. For small column sparsities $k$ relative to the projection dimension $d$, the estimated compressed left factors $\tilde{W}$ are good approximations to the desired solutions $PW$. Encouragingly, we find that the recovered solutions $\hat{W}=\mathcal{R}(\tilde{W})$ are typically only slightly worse in approximation error than $\mathcal{R}(PW)$, the solution recovered when the projection of $W$ is known exactly.
Thus, we perform almost as well as the idealized setting where we are \emph{given} the correct factorization $(PW)H$.

\vspace{-0pt}
\subsection{NMF on Gene Expression Data}
\label{sec:experiments-nmf}

\begin{table*}
  \caption{Summary of DNA microarray gene expression datasets, along with runtime (seconds) for each stage of the NMF pipeline on compressed data. \textsc{Factorize-Recover} runs only $r$ instances of sparse recovery, as opposed to the $m$ instances used by the alternative, \textsc{Recover-Factorize}.}
  \label{table:gene-timing}
  \centering
  %\resizebox{0.8\columnwidth}{!}
  {
  \begin{tabular}{@{}rccrrcrr@{}}
  \toprule
  \multirow{2}{*}{Dataset} &
\multirow{2}{*}{\# Samples} &
\multirow{2}{*}{\# Features} &
   \multicolumn{2}{c}{\textsc{Recover-Fac.}} & \phantom{a} & \multicolumn{2}{c}{\textsc{Fac.-Recover}} \\
   \cmidrule(l){4-5} \cmidrule(l){7-8}
     & &     & Recovery & NMF && NMF & Recovery \\
  \midrule
  CNS tumors & 266 & 7,129 & 76.1 & 2.7 && 0.6 & 5.4 \\
  Lung carcinomas & 203 & 12,600 & 78.8 & 4.0 && 0.8 & 9.3 \\
  Leukemia & 435 & 54,675 & 878.4 & 39.6 && 6.9 & 55.0 \\
  \bottomrule
  \end{tabular}
  }
 % \vspace{4pt}
\end{table*}

\begin{figure*}
	\centering
	\includegraphics[width=\textwidth]{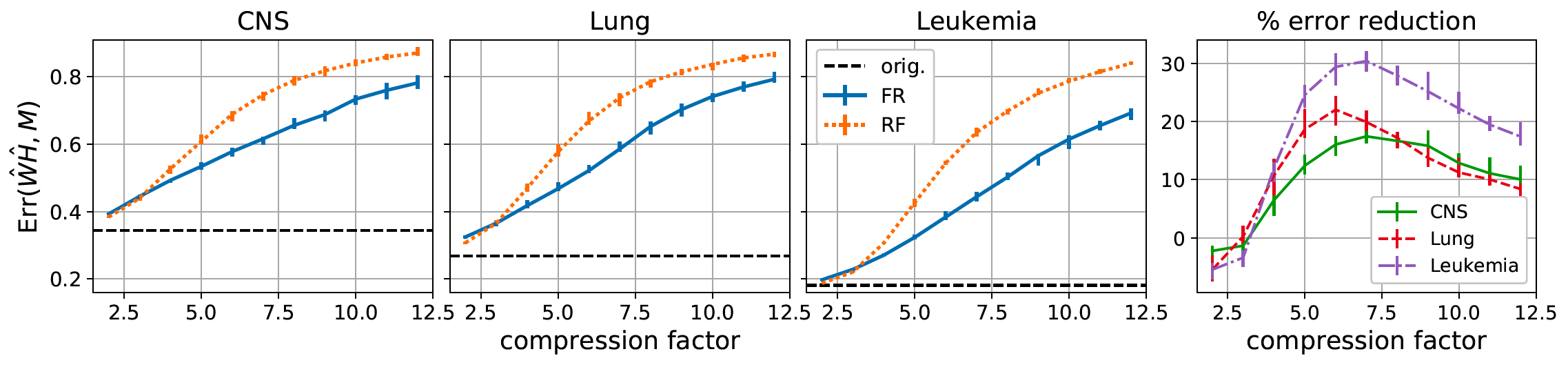}
	\caption{Normalized reconstruction errors $\|\hat{W}\hat{H} - M\|_F / \|M\|_F$ for NMF on gene expression data with varying compression factors $n/d$. \textbf{FR} (blue, solid) is \textsc{Factorize-Recover}, \textbf{RF} (orange, dotted) is \textsc{Recover-Factorize}. The horizontal dashed line is the error when $M$ is decomposed in the original space. Perhaps surprisingly, when $n/d > 3$, we observe a reduction in reconstruction error when compressed data is first factorized. See the text for further discussion.}
	\label{fig:gene-plot}
\end{figure*}
%\vspace{-4pt}
NMF is a commonly-used method for clustering gene expression data, yielding interpretable factors in practice \citep{gao2005improving,kim2007sparse}. 
%Since the dictionary factors obtained by NMF on gene expression data have been empirically observed to be sparse, this domain is well-suited for testing our proposed NMF approach. 
In the same domain, compressive sensing techniques have emerged as a promising approach for efficiently measuring the (sparse) expression levels of thousands of genes using compact measurement devices \citep{parvaresh2008recovering,dai2008compressive,cleary2017efficient}.\footnote{
The measurement matrices for these devices can be modeled as sparse binary matrices since each dimension of the acquired signal corresponds to the measurement of a small set of gene expression levels.
}
%Data from this domain is therefore well-suited for evaluating sparse matrix factorization on compressed data.
We evaluated our proposed NMF approach on gene expression datasets targeting three disease classes: embryonal central nervous system tumors \citep{pomeroy2002prediction}, lung carcinomas \citep{bhattacharjee2001classification}, and leukemia \citep{mills2009microarray} (Table~\ref{table:gene-timing}). Each dataset is represented as a real-valued matrix where the $i$th row denotes expression levels for the $i$th gene across each sample.\\

\noindent\textbf{Experimental Setup.}\hspace{0.25em} For all datasets, we fixed rank $r=10$ following previous clustering analyses in this domain \citep{gao2005improving,kim2007sparse}. 
For each data matrix $M\in\mathbb{R}^{n\times m}$, we simulated compressed measurements $\tilde{M}\in\mathbb{R}^{d \times m}$ by projecting the feature dimension: $\tilde{M} = PM$.
%The measurement matrix $\tilde{M}$ was used as input to the \textsc{Recover-Factorize} and \textsc{Factorize-Recover} procedures.
We ran projected gradient descent \citep{lin2007projected} for 250 iterations, which was sufficient for convergence.\\
%on our datasets.

%\paragraph{Cluster and Feature Recovery} {\color{red}[kst: TODO]}

%\vspace{-10pt}
\noindent\textbf{Computation Time.}\hspace{0.25em} Computation time for NMF on all 3 datasets  (Table~\ref{table:gene-timing}) is dominated by the cost of solving instances of the LP (\ref{eq:basis-pursuit}). %\footnote{We measured runtimes on a machine with a 2.5GHz Intel i7 CPU with 4 cores and 16GB RAM. We solved the LP problems using the SCS solver \citep{scs}.} 
As a result, \textsc{Factorize-Recover} achieves much lower runtime as it requires a factor of $m / r$ fewer calls to the sparse recovery procedure. %Note that the recovery times do not exactly scale by this factor, possibly due to differences in problem conditioning.
While fast iterative recovery procedures such as SSMP \citep{berinde2009sequential} achieve faster recovery times, we found that they require approximately $2\times$ the number of measurements to achieve comparable accuracy to LP-based sparse recovery.\\

%\vspace{-10pt}
\noindent\textbf{Reconstruction Error.}\hspace{0.25em} For a fixed number of measurements $d$, we observe that the \textsc{Factorize-Recover} procedure achieves lower approximation error than the alternative method of recovering prior to factorizing (Figure~\ref{fig:gene-plot}). 
While this phenomenon is perhaps counter-intuitive, it can be understood as a consequence of the sparsifying effect of NMF.
Recall that for NMF, we model each column of the compressed data $\tilde{M}$ as a nonnegative linear combination of the columns of $\tilde{W}$.
Due to the nonnegativity constraint on the entries of $\tilde{W}$, we expect the average sparsity of the columns of $\tilde{W}$ to be at least that of the columns of $\tilde{M}$.
Therefore, if $\tilde{W}$ is a good approximation of $PW$, we should expect that the sparse recovery algorithm will recover the columns of $W$ at least as accurately as the columns of $M$, given a fixed number of measurements.
%We compared the reconstruction error in the Frobenius norm achieved by \textsc{Recover-Factorize} versus \textsc{Factorize-Recover} at varying compression levels, $n/d$.
%The error was computed relative to the original data $X$.
% Figure~\ref{fig:gene-plot} summarizes our results. 
%We found that for a fixed number of measurements, NMF in the compressed domain achieved better reconstruction error than in the uncompressed domain at compression factors above $3\times$, with a decrease in error of up to $30\%$ on the Leukemia dataset. 
%We observed a similar accuracy gap when testing variants of the factorization and recovery procedures: in particular, when performing sparse recovery using the Lasso instead of LP-based recovery, and when regularizing NMF with an $\ell_1$ penalty on $W$.
%The existence of such a substantial accuracy gap is perhaps surprising---we investigate this phenomenon further in our experiments on synthetic data.

\subsection{Tensor Decomposition on EEG Time Series Data}
\label{sec:experiments-td}
%\vspace{-4pt}

EEG readings are typically organized as a collection of time series, where each series (or channel) is a measurement of electrical activity in a region of the brain. Order-3 tensors can be derived from this data by computing short-time Fourier transforms (STFTs) for each channel, yielding a tensor where each slice is a time-frequency matrix. 
%Tensor decomposition can then be applied to cluster patterns of neural activity \citep{cong2015tensor}. 
%In this domain, random projections are used as a computationally cheap data compression technique for reducing storage and bandwidth costs in portable monitoring devices \citep{abdulghani2012compressive,zhang2013compressed,craven2015compressed}.
%The CHB-MIT Scalp EEG Database \citep{shoeb2010application} is a collection of scalp EEG recordings of children susceptible to epileptic seizures.
We experimented with tensor decomposition on a compressed tensor derived from the CHB-MIT Scalp EEG Database \citep{shoeb2010application}.
In the original space, this tensor has dimensions
$27804\times303\times23$ ($\text{time} \times \text{frequency} \times \text{channel}$), corresponding to 40 hours of data (see the Appendix ~%\ref{appendix:experiments} 
for further preprocessing details). 
The tensor was randomly projected along the temporal axis. 
We then computed a rank-10 non-negative CP decomposition of this tensor using projected Orth-ALS~\citep{sharan2017orthogonalized}.\\
%We expect that this data should exhibit the desired sparsity property since EEG signals typically exhibit well-localized time-frequency structure.
%We then computed a rank-10 non-negative decomposition of this tensor using projected CP-ALS \citep{kolda2009tensor}.

\noindent\textbf{Reconstruction Error.}\hspace{0.25em} At projection dimension $d=1000$, we find that \textsc{Factorize-Recover} achieves comparable error to \textsc{Recover-Factorize} (normalized Frobenius error of $0.83$ vs. $0.82$). 
However, RF is three orders of magnitude slower than FR on this task due to the large number of sparse recovery invocations required (once for each frequency bin/channel pair, or $303\times23 = 6969$).\\

\noindent{\textbf{Factor Interpretability.}}\hspace{0.25em} The EEG time series data was recorded from patients suffering from epileptic seizures \citep{shoeb2010application}. 
We found that the tensor decomposition yields a factor that correlates with the onset of seizures (Figure~\ref{fig:seizure-factor}). At $5\times$ compression, the recovered factor qualitatively retains the interpretability of the factor obtained by decomposing the tensor in the original space.

\begin{figure*}
\centering
\includegraphics[width=0.9\textwidth]{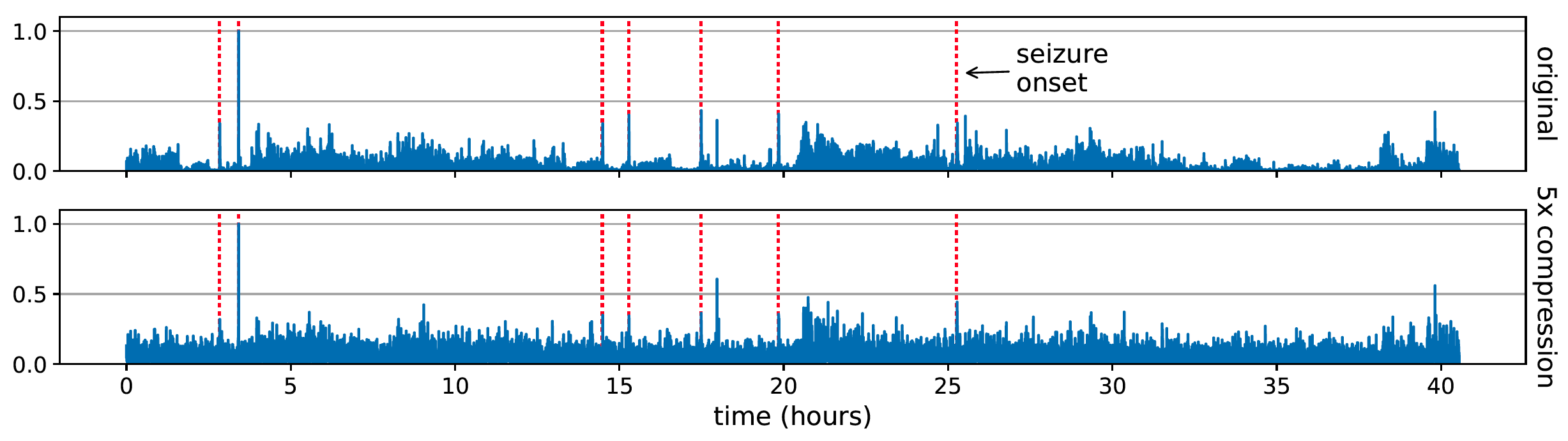}
\caption{Visualization of a factor from the tensor decomposition of EEG data that correlates with the onset of seizures in a patient (red dotted lines). The factor recovered from a $5\times$ compressed version of the tensor (bottom) retains the peaks that are indicative of seizures.}
\label{fig:seizure-factor}
\end{figure*}
\vspace{-2pt}
\section{Discussion and Conclusion}
\vspace{-2pt}
We briefly discuss our theoretical results on the uniqueness of sparse matrix factorizations in the context of dimensionality reduction via random projections. Such projections are known to preserve geometric properties such as pairwise distances \citep{kane2014sparser} and even singular vectors and singular vectors \citep{halko2011finding}. Here, we showed that maximally sparse solutions to certain factorization problems are preserved by sparse binary random projections. Therefore, our results indicate that random projections can also, in a sense, preserve certain solutions of non-convex, NP-Hard problems like NMF \citep{vavasis2009complexity}.

%\vspace{-2pt}
%\section{Conclusion}
%\vspace{-4pt}
To conclude, in this work we analyzed low-rank matrix and tensor decomposition on compressed data. Our main theoretical contribution is a novel uniqueness result for the matrix factorization case that relates sparse solutions in the original and compressed domains. We provided empirical evidence on real and synthetic data that accurate recovery can be achieved in practice. More generally, our results in this setting can be interpreted as the unsupervised analogue to previous work on supervised learning on compressed data. A promising direction for future work in this space is to examine other unsupervised learning tasks which can directly performed in the compressed domain by leveraging sparsity.

%A natural direction for future work is to extend this to other matrix and tensor decomposition models such as Tucker decompositions \cite{kolda2009tensor}.

%We proposed a simple framework for learning sparse models by combining ideas from dimension reduction and compressive sensing, and theoretical and empirical results illustrate the viability of our approach for NMF  and tensor decomposition. A natural direction of future work is to extend these theoretical and experimental results to more models and applications that require learning sparse models.

\section*{Acknowledgments}
We thank our anonymous reviewers for their valuable feedback on earlier versions of this manuscript.
This research was supported in part by affiliate members and other supporters of the Stanford DAWN project---Ant Financial, Facebook, Google, Intel, Microsoft, NEC, SAP, Teradata, and VMware---as well as Toyota Research Institute, Keysight Technologies, Northrop Grumman, Hitachi, NSF awards AF-1813049 and CCF-1704417, an ONR Young Investigator Award N00014-18-1-2295, and Department of Energy award DE-SC0019205.

\bibliographystyle{unsrtnat}
\bibliography{references}

\newpage
\onecolumn
\appendix
\section{Supplementary Experimental Results}
\label{appendix:experiments}

\subsection{Additional Experimental Details}

\noindent\textbf{Non-negative Matrix Factorization.}\hspace{0.25em}
We optimize the following NMF objective:
\begin{align*}
\underset{W,H}{\mathrm{minimize}} &\quad  \| M - WH \|_F^2 \\
\text{subject to} &\quad  W_{ij}\geq 0, H_{jk}\geq 0 \quad \forall i, j, k \nonumber
\end{align*}
for $M\in\mathbb{R}^{n\times m}$, $W\in\mathbb{R}^{n\times r}$, $H\in\mathbb{R}^{r\times m}$.
We minimize this objective with alternating non-negative least squares using the projected gradient method~\citep{lin2007projected}.
In our experiments, we initialized the entries of $W$ and $H$ using the absolute value of independent mean-zero Gaussian random variables with variance $\frac{1}{nm} \sum_{i,j} M_{ij} / r$.
We use the same step size rule as in \citet{lin2007projected} with an initial step size of $1$.

\noindent\textbf{Sparse PCA.}\hspace{0.25em}
We optimize the following sparse PCA objective:
\begin{align*}
\underset{W,H}{\mathrm{minimize}} & \quad \frac{1}{2} \| M - WH\|_F^2 + \lambda \|W\|_1 \\
\text{subject to} &\quad \sum_{j=1}^m H_{ij}^2 = 1 \text{ for } 1 \leq i \leq r.
\end{align*}
The hyperparameter $\lambda \geq 0$ controls the degree of sparsity of the factor $W$.
We optimize this objective via alternating minimization with LARS, using the open source \texttt{SparsePCA} implementation in \texttt{scikit-learn 0.20.0} with its default settings.\footnote{\url{https://scikit-learn.org}}
Here, the factors $W$ and $H$ are initialized deterministically using the truncated SVD of $M$.

\noindent\textbf{Non-negative CP Tensor Decomposition.}\hspace{0.25em}
We optimize the following objective:
\begin{align*}
\underset{A,B,C}{\mathrm{minimize}} & \quad  \| T - \sum_{\ell=1}^r A_\ell \otimes B_\ell \otimes C_\ell \|_F^2 \\
\text{subject to}& \quad A_{i\ell} \geq 0, B_{j\ell} \geq 0, C_{k\ell} \geq 0 \quad \forall i,j,k,\ell
\end{align*}
for $T\in\mathbb{R}^{n\times m_1\times m_2}$, $A\in\mathbb{R}^{n\times r}$, $B\in\mathbb{R}^{m_1 \times r}$, $C\in\mathbb{R}^{m_2 \times r}$.
We optimize this objective using a variant of Orthogonalized Alternating Least Squares, or Orth-ALS~\citep{sharan2017orthogonalized} where the entries of each iterate is projected after each update such that their values are non-negative.
We initialize the entries of $A$, $B$, and $C$ using the absolute value of independent standard normal random variables.
For the EEG time series experiment, we use the open source \texttt{MATLAB} implementation of Orth-ALS,\footnote{\url{http://web.stanford.edu/~vsharan/orth-als.html}} modified to incorporate the non-negativity constraint.

\subsection{NMF Reconstruction Error and Projection Matrix Column Sparsity ($p$)}
\label{app:projection-sparsity}

We investigated the trade-off between reconstruction error (as measured by normalized Frobenius loss) and the sparsity parameter $p$ of the binary random projections $P$. Recall that $P\in\{0,1\}^{d\times n}$ is a randomly sampled sparse binary matrix where $p$ distinct entries in each column are selected uniformly at random and set to 1. 
In Figure~\ref{fig:nmf-sparsity}, we plot the normalized reconstruction error achieved by NMF using \textsc{Factorize-Recover} on the lung carcinoma gene expression dataset \citep{bhattacharjee2001classification} at a fixed compression level of 5.
Since we observed that the cost of sparse recovery increases roughly linearly with $p$, we aimed to select a small value of $p$ that achieves good reconstruction accuracy.
We found that the setting $p=5$ was a reasonable choice for our experiments.

\begin{figure}[H]
  \centering
  \includegraphics[width=0.5\textwidth]{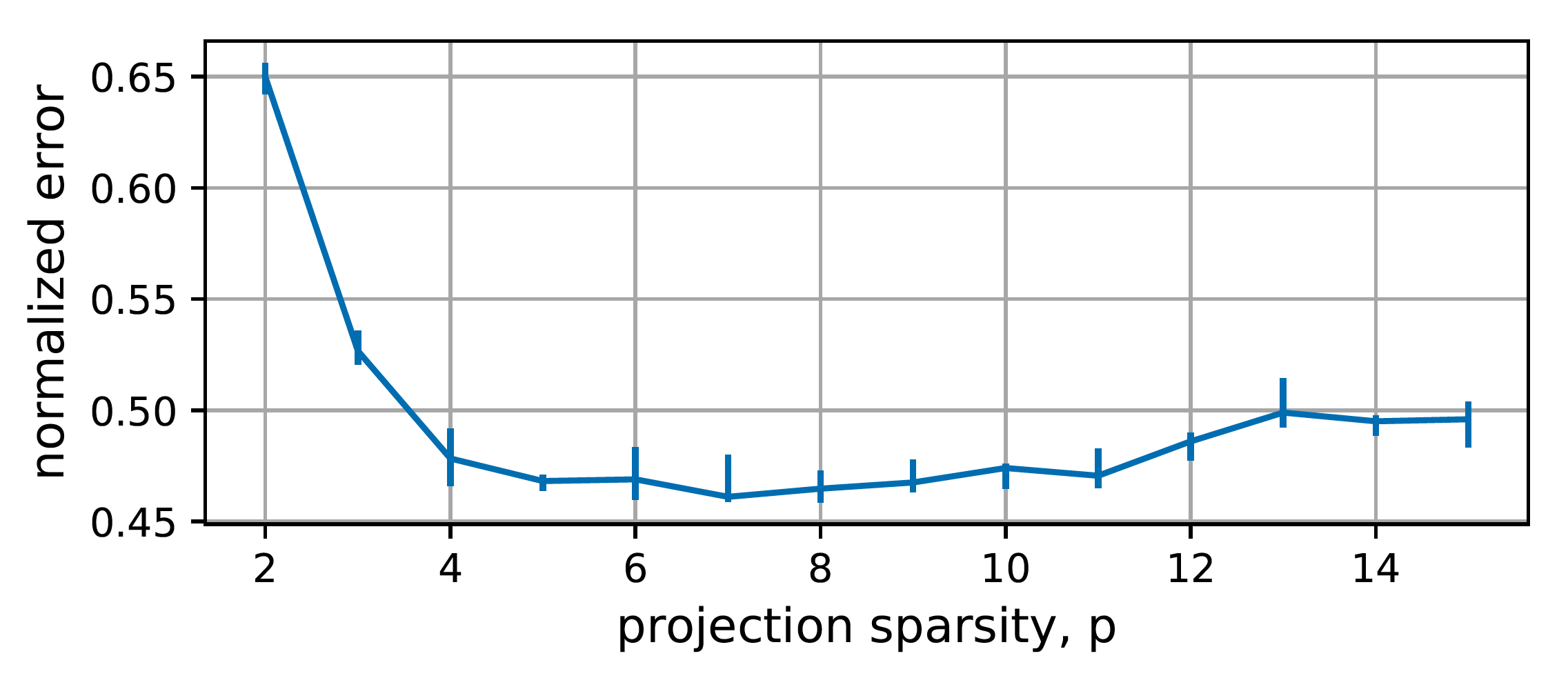}
  \caption{NMF reconstruction error vs. projection matrix column sparsity.}
  \label{fig:nmf-sparsity}
\end{figure}

%\subsection{Comparing Factors in Compressed Space for NMF}

%Recall that Theorem~\ref{thm:nmf} established conditions under which the sparsest possible factorization of $X$ (in terms of the left factors) maps, under the projection $P$, to the sparsest possible factorization of $PX$. To validate if this theoretical result is reflected in practice, we compared the left factors obtained by NMF on $\tilde{M}=PX$ with the projection $PW_*$ of the left factor $W_*$ obtained from NMF on $X$. We compared factors by computing (via maximum bipartite matching) the mean Pearson correlation between columns of $\tilde{W}$ and $PW_*$, and between those of $\mathcal{R}(\tilde{W})$ and $W_*$. In Figure~\ref{fig:nmf-corr}, we see that the median correlation remains relatively high with increasing compression, as we would expect from the theory, and that much of the observed error is due to inaccuracy in the recovery step as the number of measurements decreases.

%\begin{figure}
%	\centering
%	\includegraphics[width=3 in]{figs/nmf-corr.pdf}
%	\caption{Median Pearson correlations between columns of $\tilde{W}$ and $PW_*$ (``F only''; see text for definition), and between $\mathcal{R}(\tilde{W})$ and $W_*$ (``FR''). Most of the observed error is due to inaccuracy in sparse recovery as the number of measurements decreases.}
%	\label{fig:nmf-corr}
%\end{figure}

\subsection{Preprocessing of EEG Data}

Each channel is individually whitened with a mean and standard deviation estimated from segments of data known to not contain any periods of seizure. The spectrogram is computed with a Hann window of size 512 (corresponding to two seconds of data). The window overlap is set to 64. In order to capture characteristic sequences across time windows, we transform the spectrogram by concatenating groups of sequential windows, following \citep{shoeb2010application}. We concatenate groups of size three.

%\iffalse
\begin{figure}[H]
  \centering
  \includegraphics[width=3 in]{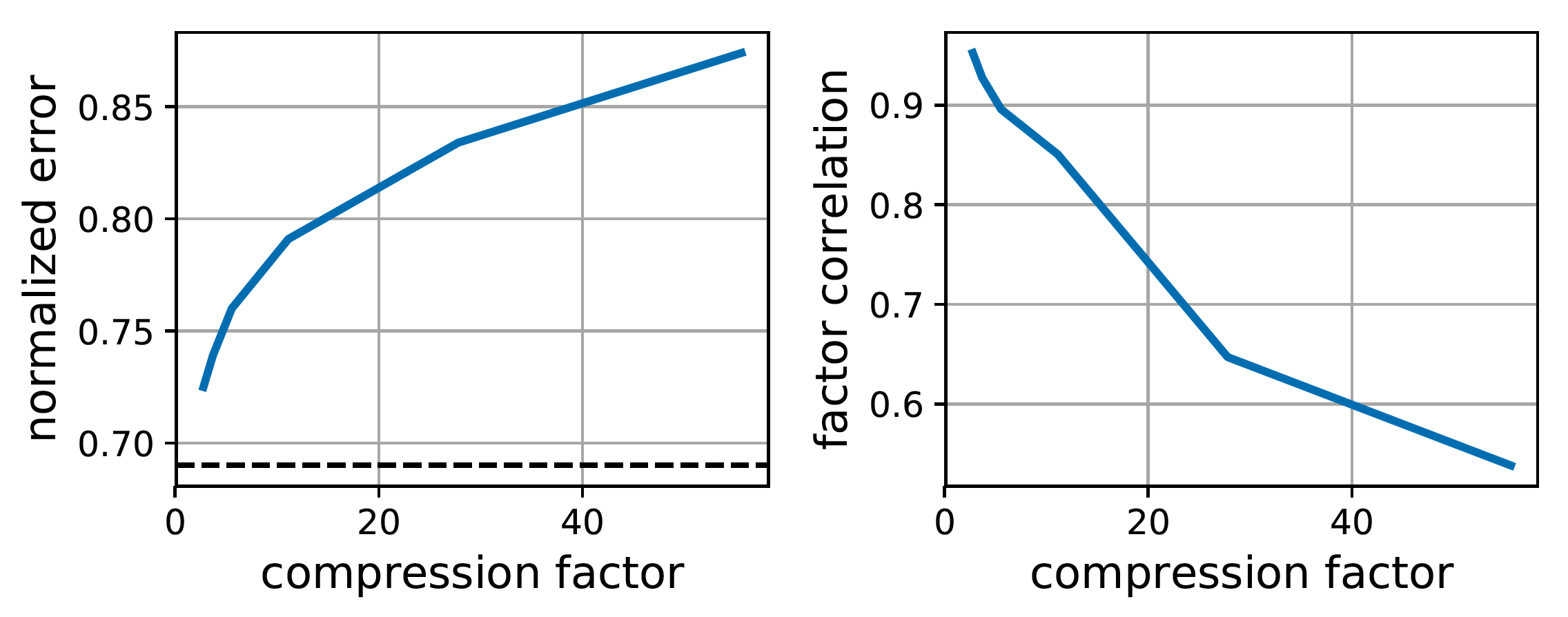}
  \caption{Accuracy of tensor decomposition on compressed EEG data. \textbf{Left:} Normalized reconstruction error; dashed line indicates baseline reconstruction error on original data. \textbf{Right:} Median Pearson correlations between recovered factors and factors computed from original data.}
  \label{fig:td}
\end{figure}
%\fi

\subsection{Tensor Decomposition of Compressed EEG Data}

In Figure~\ref{fig:td} (left), we plot the normalized Frobenius errors of the recovered factorization against the compression factor $n/d$.
Due to the sparsity of the data, we can achieve over $10\times$ compression for the cost of a $10\%$ increase in reconstruction error relative to the baseline decomposition on the uncompressed data, or approximately $28\times$ compression for a $15\%$ increase. 
\textsc{Recover-Factorize} (RF) achieves slightly lower error at a given projection dimension: at $28\times$ compression ($d=1000$), RF achieves a normalized error of $0.819$ vs. $0.834$ for \textsc{Factorize-Recover}. 
However, RF is three orders of magnitude slower than \textsc{Factorize-Recover} on this dataset due to the large number of calls required to the sparse recovery algorithm (once for each frequency bin/channel pair, or $303\times 23 = 6969$) to fully recover the data tensor.
Due to the computational expense of recovering the full collection of compressed time series, we did not compare RF to FR over the full range of compression factors plotted in the figure for FR.

Figure~\ref{fig:td} (right) shows the median Pearson correlations between the columns of the recovered temporal factor and those computed on the original, uncompressed data (paired via maximum bipartite matching).
At up to $10\times$ compression, the recovered temporal factors match the temporal factors obtained by factorization on the uncompressed data with a median correlation above $0.85$.
%At $5\times$ compression, the recovered temporal factors match the temporal factors obtained by tensor decomposition on the uncompressed data with a median Pearson correlation of $0.90$. 
%This decreases to $0.78$ and $0.53$ correlation at $10\times$ and $20\times$ compression respectively. 
%The median correlation for the other (unprojected) factors is at least $0.98$ for all three compression levels. 
Thus, compressed tensor factorization is able to successfully recover an approximation to the factorization of the uncompressed tensor.
%Thus, at $5\times$ compression, we still achieve accurate recovery of the compressed factor.

%\paragraph{Detection of seizures.} The original goal of the work by \citep{shoeb2010application} was to train patient-specific classifiers to detect the onset of epileptic seizures. We find that a tensor decomposition of the time series yields a factor that correlates with the onset of seizures, as shown in Fig.~\ref{fig:seizure-factor}. This example illustrates the tradeoff between the compression factor and the fidelity of the recovered modes.

%\input{nmf_syn}
%\input{nmf_sparsity}
\section{Proof of Theorem \ref{thm:nmf}: Uniqueness for NMF}\label{sec:nmf_app}

We follow the outline from the proof sketch. Recall that our goal will be to prove that the columns of $\tilde{W}$ are the sparsest vectors in the column space of $\tilde{W}$. For readability proofs of some auxiliary lemmas appears later in Section \ref{sec:nmf_add}

As mentioned in the proof sketch, the first step is part (a)---showing that if we take any subset $S$ of the columns of $\tilde{W}$, then the number of rows which have non-zero entries in at least one of the columns in $S$ is large. Lemma \ref{lem:exp2} shows that the number of rows which are have at least one zero entry in a subset $S$ of the columns of $W$ columns proportionately with the size of $S$. The proof proceeds by showing that choosing $B$ such that each column has $k$ randomly chosen non-zero entries ensures expansion for $B$ with high probability, and we have already ensured expansion for $P$ with high probability.
%\vspace{-5pt}
\begin{restatable}{lem}{expii}\label{lem:exp2}
	For any subset $S$ of the columns of $\tilde{W}$, define $N(S)$ to be the subset of the rows of $\tilde{W}$ which have a non-zero entry in at least one of the columns in $S$. Then for every subset $S$ of columns of $\tilde{W}$,  $|N(S)|\ge\min\{16|S|kp/25,d/200\}$ with failure probability $re^{-\beta k}/n+(1/n^5)$.
\end{restatable}
%\vspace{-8pt}
We now prove the second part of the argument---that any linear combination of columns in $S$ cannot have much fewer non-zero entries than $N(S)$, as the probability that many of the non-zero entries get canceled is zero. Lemma \ref{lem:null2} is the key to showing this. Define a vector $x$ as \emph{fully dense} if all its entries are non-zero.
%\vspace{-5pt}
\begin{lem}\label{lem:null2}
	For any subset $S$ of the columns of $\tilde{W}$, let $U$ be the submatrix of $\tilde{W}$ corresponding to the $S$ columns and $N(S)$ rows. Then with probability one, every subset of the rows of $U$ of size at least $|S|p$ does not have any fully dense vector in its right null space. 
\end{lem}
%\vspace{-15pt}
\begin{proof}
	Without loss of generality, assume that $S$ corresponds to the first $|S|$ columns of $\tilde{W}$, and $N(S)$ corresponds to the first $|N(S)|$ rows of $\tilde{W}$. We will partition the rows of $U$ into $t$ groups $\{\mathcal{G}_1,\dots, \mathcal{G}_{t}\}$. Each group will have size at most $p$. To select the first group, we choose any entry $y_1$ of $Y$ which appears in the first row of $U$. For example, if the first column of $W$ has a one in its first row, and $P(1,1)=1$, then the random variable $Y_{1,1}$ appears in the first row of $U$. Say we choose $y_1=Y_{1,1}$. We then choose $\mathcal{G}_1$ to be the set of all rows where $y_1$ appears. We then remove the set of rows $\mathcal{G}_1$ from $U$. To select the second group, we pick any one of the remaining rows, and choose any entry $y_2$ of $Y$ which appears in that row of $U$. $\mathcal{G}_2$ is the set of all rows where $y_2$ appears. We repeat this procedure to obtain $t$ groups, each of which will have size at most $p$ as every variable appears in $p$ columns. Hence any subset of rows of size at least $|S|p$ must correspond to at least $|S|$ groups.
	
	% Consider the first row $U_1$ of $U$. Consider any fully dense vector $x\in\mathbb{R}^{|S|}$. As $U_1$ contains at least one non-zero entry, with probability one over the randomness in the entries of $U_1$, $x$ is not orthogonal to $U_1$. Choose any entry $y_1$ of the matrix $Y$ which appears in the row $U_1$. The set of all columns in which $y_1$ appears comprise our first group $\mathcal{G}_1$ of columns. We have shown that $\rank(\mathcal{N}_1)=\rank(\mathcal{N}_0)-1=t-1$.
	Let $\mathcal{N}_j$ be the right null space of the first $j$ groups of rows. We define $\mathcal{N}_0=\mathbb{R}^{|S|}$. We will now show that either $\rank(\mathcal{N}_i)=|S|-i$ or $\mathcal{N}_i$ does not contain a fully dense vector. We prove this by induction. Consider the $j$th step, at which we have $j$ groups $\{\mathcal{G}_1, \dots, \mathcal{G}_j\}$. By the induction hypothesis, either $\mathcal{N}_j$ does not contain any fully dense vector,  or $\rank(\mathcal{N}_j)=|S|-j$. If $\mathcal{N}_j$ does not contain any fully dense vector, then we are done as this implies that $\mathcal{N}_{j+1}$ also does not contain any fully dense vector. Assume that  $\mathcal{N}_j$ contains a fully dense vector $x$. Choose any row $U_c$ which has not been already been assigned to one of the sets. By the following elementary proposition, the probability that $x$ is orthogonal to $U_c$ is zero. We provide a simple proof in Section \ref{sec:nmf_add}.
	\begin{restatable}{lem}{ele}
		Let $v=(v_1,\dots,v_n)\in \mathbb{R}^n$ be a vector of $n$ independent random variables drawn from some continuous distribution. For any subset $S\subseteq \{1,\dots,n\}$, let $v(S)\in\mathbb{R}^{|S|}$ refer to the subset of $v$ corresponding to the indices in $S$. Consider $t$ such subsets $S_1,\dots, S_t$. Let each set $S_i$ defines some linear relation $\alpha_{S_i}^Tv(S_i)=0$, for some $\alpha_{S_i}\in \mathbb{R}^{|S_i|}$ where $\|\alpha_{S_i}\|=1$ and each entry of the vector $\alpha_{S_i}$ is non-zero on the variables in the set $S_i$. Assume that the variable $v_{i}$ appears in the set $S_i$. Then the probability distribution of the set of variables $\{v_{t+1},\dots,v_{n}\}$ conditioned on the linear relations defined by $S_1,\dots, S_t$ is still continuous, and hence any linear combination of the set of variables $\{v_{t+1},\dots,v_{n}\}$ has zero probability of being zero.
	\end{restatable}

	If $\mathcal{N}_{j}$ contains a fully dense vector, then with probability one, $\rank(\mathcal{N}_{j+1})=\rank(\mathcal{N}_{j})-1=n-j-1$. This proves the induction argument. Therefore, with probability one, for any $t\ge |S|$, either  $\rank(\mathcal{N}_{t})=0$ or $\mathcal{N}_{t}$ does not contain a fully dense vector and Lemma \ref{lem:null2} follows. %Hence the left nullspace of $U$ does not contain a fully dense vector.
\end{proof}
%\vspace{-10pt}
We now complete the proof of Theorem \ref{thm:nmf}. Note that the columns of $\tilde{W}$ have at most $kp$ non-zero entries, as each column of $P$ has $p$-sparse. Consider any set $S$ of columns of $\tilde{W}$.  Consider any linear combination $v\in \mathbb{R}^d$ of the set $S$ columns, such that all the combination weights $x\in\mathbb{R}^{|S|}$ are non-zero. By Lemma \ref{lem:exp2}, $|N(S)|\ge \min\{16|S|kp/25,d/200\}$ with failure probability $re^{-\beta k}/n+(1/n^5)$. We claim that $v$ has more than $|N(S)|-|S|p$ non zero entries. We prove by contradiction. Assume that $v$ has $|N(S)|-|S|$ or fewer non zero entries. Consider the submatrix $U$ of $\tilde{W}$ corresponding to the $S$ columns and $N(S)$ rows. If $v$ has $|N(S)|-|S|p$ or fewer non zero entries, then there must be a subset $S'$ of the $|N(S)|$ rows of $U$ with $|S'|=|S|p$, such that each of the rows in $S'$ has at least one non-zero entry, and the fully dense vector $x$ lies in the right null space of $S'$. But by Lemma \ref{lem:null2}, the probability of this happening is zero. Hence $v$ has more than $|N(S)|-|S|p$ non zero entries. Lemma \ref{lem:nmf_comp} obtains a lower bound on $|N(S)|-|S|p$ using simple algebra.
\begin{restatable}{lem}{nmfcom}\label{lem:nmf_comp}
	$|N(S)|-|S|p \ge 6kp/5$ for $|S|>1$ for $d\ge 400p(r+k)$.
\end{restatable}
%\vspace{-10pt}	
Hence any linear combination of more than one column of $\tilde{W}$ has at least $6kp/5$ non-zero entries with failure probability $re^{-\beta k}/n$. Hence the columns of $\tilde{W}$ are the sparsest vectors in the column space of $\tilde{W}$ with failure probability $re^{-\beta k}/n+(1/n^5)$.

\subsection{Additional Proofs for Uniqueness of NMF}\label{sec:nmf_add}

	\begin{restatable}{lem}{spaceeq}\label{lem:sp_eq}
		If $H$ is full row rank, then the column spaces of $\tilde{W}$ and ${W}'$ are equal.
	\end{restatable}
\begin{proof}
	We will first show that the column space of $\tilde{M}$ equals the column space of $\tilde{W}$.  Note that the column space of $\tilde{M}$ is a subspace of the column space of $\tilde{W}$. As $H$ is full row rank, the rank of the column space of $\tilde{M}$ equals the rank of the column space of $\tilde{W}$. Therefore, the column space of $\tilde{M}$ equals the column space of $\tilde{W}$. 
	
	By the same argument, for any alternative factorization $\tilde{M}=W'H'$, the column space of $W'$ must equal the column space of $\tilde{M}$---which equals the column space of $\tilde{W}$. As the column space of $W'$ equals the column space of $\tilde{W}$, therefore $W'$ must lie in the column space of $\tilde{W}$.
\end{proof}

\expii*
\begin{proof}
	We first show a similar property for the columns of $W$, and will then extend it to the columns of $\tilde{W}=PW$. We claim that for every subset of $S$ columns of $W$, $|N(S)|\ge\min\{4|S|k/5,n/200\}$ with failure probability $re^{-\beta k}/n$.

%	We first claim that the number of neighbors in $V$ of any set $S$ of nodes in $U$ is of size at least $\min\{4|S|k/5,n/200\}$ with failure probability $re^{-\beta k}/n$. 
	
	To verify, consider a bipartite graph $T$ with $r$ nodes on the left part $U_1$ corresponding to the $r$ columns of $W$, and $n$ nodes on the right part $V$ corresponding to the $n$ rows or indices of each factor. The $i$th node in $U_1$ has an edge to $k$ nodes in $V$ corresponding to the non-zero indices of the $i$th column of $W$. Note that $|N(S)|$ is the neighborhood of the set of nodes $S$ in $G$. From Part 1 of Lemma \ref{expander}, the graph $G$ is a $(\gamma_1 r,4/5)$ expander with failure probability $re^{-\beta k}/n$ for $\gamma_1 = n/(rke^5)$ and a fixed constant $\beta>0$. 
	
	\begin{restatable}{lem}{expander}\label{expander}
		Randomly choose a bipartite graph $G$ with $n_1$ vertices on the left part $U$ and $n_2$ vertices on the right part $V$ such that every vertex in $U$ has degree $D$. Then,
		\begin{enumerate}
			\item  For every $n_1, n_2, n_1<n_2$, $G$ is a $(\gamma n_1, 4/5)$ expander for $D\ge c$ for some fixed constant $c$ and $\gamma n_1 =\frac{n_2}{De^{5}}$ except with probability $n_1e^{-\beta D}/n_2$ for a fixed constant $\beta>0$.
			\item For every $n_1, n_2, n_2 <n_1$, $G$ is a $(\gamma n_1, 4/5)$ expander for $D\ge c\log n_1$ for some fixed constant $c$ and $\gamma n_1 =\frac{n_2}{De^{5}}$ except with probability $(1/n_1)^5$.
		\end{enumerate}
	\end{restatable}
	
	As $G$ is a $(\gamma_1 r,4/5)$ expander, every set of $|S|\le \gamma_1 r$ nodes has at least $4|S|k/5$ neighbors. A set of size $|S| >\gamma_1 r$ nodes, must include a subset of size $\gamma_1 r$ which has $4n/(5e^5)\ge n/200$ neighbours, and hence every set of size $|S|>\gamma_1 r$ has at least $n/200$ neighbors. Therefore, for every subset of $S$ columns, $|N(S)|\ge\min\{4|S|k/5,n/200\}$ with failure probability $re^{-\beta k}/n$.
	
	We will now extend the proof to show the necessary property for $\tilde{W}$. After the projection step, the $n$ indices are projected to $d$ dimensions, and the projection matrix is a $(\gamma_2 n, 4/5)$ expander with $\gamma_2 = d/(nke^5)$. We can now consider a tripartite graph, by adding a third set $U_2$ with $d$ nodes. We add an edge from a node $i$ in $V$ to node $j$ in $U_2$ if $P(j,i)=1$. For any subset $S$ of columns of $\tilde{W}$, $N(S)$ are the set of nodes in $U_2$ which are reachable from the nodes $S$ in $U_1$. 
	
	With failure probability $(1/n^5)$, the projection matrix $P$ is a $(\gamma_2 n, 4/5)$ expander with $\gamma_2 = d/(npe^5)$. Therefore every subset of size $t$ in $V$ has at least $\min\{4tp/5,d/200\}$ neighbors in $W$. By combining this argument with the fact that every set of $S$ nodes in $U$, has at least $\min\{4|S|k/5,n/200\}$ neighbors with failure probability $re^{-\beta k}/n$, it follows that for every subset of $S$ columns of $\tilde{W}$,  $|N(S)|\ge\min\{16|S|kp/25,d/200\}$ with failure probability $re^{-\beta k}/n+(1/n^5)$.
\end{proof}

\ele*
%\begin{proof}
%%	We prove by induction. For the base case, note that without any linear constraints, the set of $n$ random variables $\{v_{1},\dots,v_{n}\}$ is continuous and as the random variables $v_i$ are absolute values of independent Gaussian draws. Consider the $j$th step, when linear constraints defined by the sets $S_1,\dots,S_{j}$ have been imposed on the variables. We claim that the distribution of the set of random variables $\{v_{j+1},\dots,v_{n}\}$ is continuous after imposition of the constraints in the sets $S_1,\dots,S_j$. By the induction hypothesis, the distribution of the set of random variables $\{v_{j},\dots,v_{n}\}$ is continuous after imposition of the constraints $S_1,\dots,S_{j-1}$. Consider any satisfying assignment to the variables $\{v_j,v_{j+1},\dots,v_{n}\}$. Note that for infinitesimally small perturbations to any of the variables $\{v_{j+1},\dots,v_{n}\}$, there is still a satisfying assignment by choosing the corresponding assignment to $v_j$ which satisfies the constraint, and this can be done as $v_j$ is continuous and by our assumption all the constraints are satisfiable with some non-zero assignment to all the variables. Hence the probability distribution of the set of variables $\{v_{j+1},\dots,v_{n}\}$ is still continuous after adding the constraint $S_j$, which verifies the induction hypothesis.
%\end{proof}

\begin{proof}

We prove by induction. For the base case, note that without any linear constraints, the set of $n$ random variables $\{v_{1},\dots,v_{n}\}$ is continuous by definition. Consider the $j$th step, when linear constraints defined by the sets $S_1,\dots,S_{j}$ have been imposed on the variables. We claim that the distribution of the set of random variables $\{v_{j+1},\dots,v_{n}\}$ is continuous after imposition of the constraints $S_1,\dots,S_j$. By the induction hypothesis, the distribution of the set of random variables $\{v_{j},\dots,v_{n}\}$ is continuous after imposition of the constraints $S_1,\dots,S_{j-1}$. Suppose that the linear constraint $\alpha_{S_j}$ is satisfied for some assignment $(t_j,\dots,t_n)$ to the random variables $\{v_{j},\dots,v_{n}\}$ which appear in the constraint $S_j$. As the distribution of the variables $\{v_{j},\dots,v_{n}\}$ is continuous by our induction hypothesis, there exists some $\epsilon>0$ such that the pdf of the variables $v_i$ for $j\le i \le n$ is non-zero in the interval $[t_i-\epsilon,t_i+\epsilon]$. Let $\delta= |\alpha_{S_i}(j)|$ be the absolute value of the linear coefficients of the variable $v_j$ in $\alpha_{S_i}$. For any choice of $v_i, j+1\le i\le n$ in the interval $[t_i-\delta\epsilon/\sqrt{n},t_i+\delta\epsilon/\sqrt{n}]$, the linear constraint $\alpha_{S_i}$ can be satisfied by some choice of the variable $v_j$ in $[t_j-\epsilon,t_j+\epsilon]$. Hence the probability distribution of the set of variables $\{v_{j+1},\dots,v_{n}\}$ is still continuous after adding the constraint $S_j$, which proves the induction step.

%	We prove by induction. For the base case, note that without any linear constraints, the set of $n$ random variables $\{v_{1},\dots,v_{n}\}$ is continuous and has full support as the random variables $v_i$ are independent Gaussian. Consider the $j$th step, when linear constraints defined by the sets $S_1,\dots,S_{j}$ have been imposed on the variables. We claim that the distribution of the set of random variables $\{v_{j+1},\dots,v_{n}\}$ is continuous and has full support after imposition of the constraints $S_1,\dots,S_j$. By the induction hypothesis, the distribution of the set of random variables $\{v_{j},\dots,v_{n}\}$ is continuous and has full support after imposition of the constraints $S_1,\dots,S_{j-1}$. Note that the linear constraint $S_j$ can be satisfied for any assignment to the subset of  variables $\{v_{j+1},\dots,v_{n}\}$ which appear in the constraint $S_j$, as $v_j$ can be chosen appropriately because by the induction hypothesis it has full support conditioned on the previous constraints $S_1,\dots, S_{j-1}$. Hence the probability distribution of the set of variables $\{v_{j+1},\dots,v_{n}\}$ is still continuous and has full support after adding the constraint $S_j$.
\end{proof}

\nmfcom*
\begin{proof}
	For $2\le |S|\le  d/(128kp)$, 
	\begin{align}
	|N(S)|-|S|p&\ge (16kp/25)|S|-p|S|=30kp/25+kp(16|S|-30)/25-p|S| \nonumber\\
	&\ge 6kp/5+p\Big(k(16|S|-30)-|S|\Big) \nonumber
	\end{align} 
	For $|S|\ge 2$ and $k\ge2$, $k(16|S|-30)-|S|\ge 0$, hence $|N(S)|-|S|p\ge 6kp/5$ for $2\le |S|\le  d/(128kp)$. For $|S|>d/(128kp)$, $|N(S)|\ge d/200$. Therefore, $|N(S)|-|S|p\ge d/200-rp\ge 2kp$ for $d\ge 400p(r+k)$. 
\end{proof}

\expander*
\begin{proof}
	Consider any subset $S\subset U$ with $|S|\le \gamma n_1$. Let $\Prob(N(S)\subseteq M)$ denote the probability of the event that the neighborhood of $S$ is entirely contained in $M\subset V$. $\Prob(N(S)\subseteq M) \le \Big(\frac{|M|}{n_2}\Big)^{D|S|}$. We will upper bound the probability of $G$ not being an expander by upper-bounding the probability of each subset $S\subset U$ with $|S|\le \gamma n_1$ not expanding. Let $\Prob(\bar{S})$ denote the probability of the neighborhood of $S$ being entirely contained in a subset $M\subset V$ with $M< {\alpha} |S|D$. By a union bound,
	\begin{align}
	\Prob(G \text{ is not a $(\gamma n_1,\alpha)$ expander}) &\le \sum_{\substack{S\subset U\\ |S|\le \gamma n_1}}^{} \Prob(\bar{S})\nonumber\\
	&\le \sum_{\substack{S\subset U\\ |S|\le \gamma n_1}}^{} \sum_{\substack{M\subset V\\ M= {\alpha} |S|D}}^{}\Prob(N(S)\subseteq M)\nonumber\\
	&\le  \sum_{s=1}^{\gamma n_1}\sum_{\substack{S\subset U\\ |S|= s}}^{} \sum_{\substack{M\subset V\\ M= {\alpha} |S|D}}^{} \Big(\frac{{\alpha} |S|D}{n_2}\Big)^{D|S|}\nonumber\\
	&\le   \sum_{s=1}^{\gamma n_1} {n_1 \choose  s}{n_2 \choose {\alpha} Ds}\Big(\frac{{\alpha} Ds}{n_2}\Big)^{Ds}\nonumber
	\end{align}
Using the bound ${n \choose k}\le (ne/k)^k$, we can write, 
\begin{align}	
	\Prob(G \text{ is not a $(\gamma n_1,\alpha)$ expander})
	&\le   \sum_{s=1}^{\gamma n_1} \Big(\frac{n_1 e}{  s}\Big)^{{\alpha} s}\Big(\frac{n_2 e}{ {\alpha} Ds}\Big)^{{\alpha} Ds}\Big(\frac{{\alpha} Ds}{n_2}\Big)^{Ds}\nonumber\\
	&\le \sum_{s=1}^{\gamma n_1} \Bigg[ \Big(\frac{n_1 e}{  s}\Big)^{{\alpha} }\Big(\frac{n_2 e}{ {\alpha} Ds}\Big)^{{\alpha} D}\Big(\frac{{\alpha} Ds}{n_2}\Big)^{D} \Bigg]^s\nonumber \le \sum_{s=1}^{\gamma n_1} x_s^s\nonumber
	\end{align}
	where $x_s= \Big(\frac{n_1 e}{  s}\Big)^{ }\Big(\frac{n_2 e}{ {\alpha} Ds}\Big)^{{\alpha} D}\Big(\frac{{\alpha} Ds}{n_2}\Big)^{D} $. $x_s$ can be bounded as follows---
	\begin{align}
	x_s &= \Big(\frac{n_1 e}{  s}\Big)\Big( \frac{{\alpha} D s e^{1/{(1-\alpha)}}}{n_2} \Big)^{{(1-\alpha)} D}\nonumber\\
	&\le \Big(\frac{e}{ \gamma }\Big)\Big( \frac{{\alpha} D \gamma n_1 e^{1/(1-\alpha)}}{n_2} \Big)^{{(1-\alpha)} D}\nonumber\\
	&\le \Big(\frac{n_1 e^{1+1/(1-\alpha)}}{n_2}\Big)D{\alpha}^{{(1-\alpha)}D }\nonumber\\
	&\le \Big(\frac{n_1 e^{6}}{n_2}\Big)De^{-D/25}
	= x\nonumber
	\end{align}
	where in the last step we set $\alpha=4/5$. Hence we can upper bound the probability of $G$ not being an expander as follows---
	\begin{align}
	\Prob(G \text{ is not a $(\gamma n_1,\alpha)$ expander}) \le \sum_{s=1}^{\infty}x^s \le  \frac{x}{1-x}\nonumber
	\end{align}
	The two parts of Lemma \ref{expander} follow by plugging in the respective values for $n_1, n_2$ and $D$.
	\begin{comment}
	We now bound $x$ separately to prove the two parts of Lemma \ref{expander}.
	\begin{enumerate}
	\item For $n_1=n^c$ and $n_2 = n$, we bound $x$ as follows,
	\begin{align}
	x &=  \Big(\frac{n_1 e^{2}}{n_2}\Big)De^{-0.3D}\nonumber\\
	&\le \frac{e^2}{n^{1-c}}ke^{-0.3k} \nonumber
	\end{align}
	\item For $n_1>n_2$, we bound $x$ as follows,
	\begin{align}
	x &=  \Big(\frac{n_1 e^{2}}{n_2}\Big)re^{-0.3r}\nonumber\\
	&\le \Big(\frac{n_2}{n_1}\Big)\nonumber
	\end{align}
	for $r=10\log(n_1/n_2)$.
	
	\end{enumerate}
	\end{comment}	
\end{proof}
\section{Recovery Guarantees for Compressed Tensor Factorization using ALS-based Methods}
\label{sec:tensor_app}
We can prove a stronger result for symmetric, incoherent tensors and guarantee accurate recovery in the compressed space using the tensor power method. The tensor power method is the tensor analog of the matrix power method for finding eigenvectors. It is equivalent to finding a rank 1 factorization using the Alternating Least Squares (ALS) algorithm. Incoherent tensors are tensors for which the factors have small inner products with other. We define the incoherence $\mu=\max_{i\ne j}\{|A_i^TA_j|\}$. Our guarantees for tensor decomposition follow from the analysis of the tensor power method by \citet{sharan2017orthogonalized}. Proposition \ref{tensor_rec} shows guarantees for recovering one of the true factors, multiple random initializations can then be used for the tensor power method to recover back all the factors (see \citet{anandkumar2014guaranteed}).

% We state our guarantees for Orthogonalized ALS---a variant of ALS introduced in \citet{sharan2017orthogonalized} for which guaranteed recovery is known. Note that ALS by itself does not have any global convergence guarantees. Similar guarantees to Proposition \ref{tensor_rec} can be obtained for the tensor power method using the results in \citet{sharan2017orthogonalized} and \citet{anandkumar2014guaranteed}.

\begin{restatable}{prop}{tensor}\label{tensor_rec}
	Consider a $n$-dimensional rank $r$ tensor $T=\sum_{i=1}^{r}w_i A_i \otimes A_i \otimes A_i$. Let $c_{\max}=\max_{i \ne j} |A_i^T A_j|$ be the incoherence between the true factors and $\gamma=\frac{w_{\max}}{w_{\min}}$ be the ratio of the largest and smallest weight. Assume $\gamma$ is a constant and $ \mu \le o(r^{-2})$. Consider a projection matrix $P\in\{0,\pm1\}^{n\times d}$ where every row has exactly $p$ non-zero entries, chosen uniformly and independently at random and the non-zero entries have uniformly and independently distributed signs. We take $d=O(r^4\log r)$ and $p=O(r^2\log r)$. Let $\tilde{A}=AP$ and $\tilde{T}$ be the $d$ dimensional projection of $T$, hence $ \tilde{T}=\sum_{i=1}^{k}w_i \tilde{A}_i \otimes \tilde{A}_i \otimes \tilde{A}_i$. Then
	%\begin{enumerate}
	%	\item For the original tensor decomposition problem, if the initialization $x_0\in \mathbb{R}^n$ is chosen uniformly at random from the unit sphere, then with high probability the tensor power method converges to one of the true factors of $A$ (say the first factor $A_1$) in $\Oh( r(\log r +\log \log n))$ steps, and the estimate $A_1'$ satisfies $\normsq{A_1 - A_1'} \le \Oh( r\max\{\mu^2,1/n^2\})$. 
	 for the projected tensor decomposition problem, if the initialization $x_0\in \mathbb{R}^d$ is chosen uniformly at random from the unit sphere, then with high probability the tensor power method converges to one of the true factors of $\tilde{T}$ (say the first factor $\tilde{A}_1$) in $\Oh( r(\log r +\log \log d))$ steps, and the estimate $\tilde{A}'$ satisfies $\normsq{\tilde{A}_1 - \tilde{A}_1'} \le \Oh( r\max\{\mu^2,1/d^2\})$.
%	\end{enumerate}
\end{restatable}

\begin{proof}
	Our proof relies on Theorem 3 of \citet{sharan2017orthogonalized} and sparse Johnson Lindenstrauss transforms due to \citet{kane2014sparser}. To show Claim 2 we need to ensure that the incoherence parameter in the projected space is small. We use the Johnson Lindenstrauss property of our projection matrix to ensure this. A matrix $M$ is regarded as a Johnson Lindenstrauss matrix if it preserves the norm of a randomly chosen unit vector $x$ up to a factor of $(1\pm\epsilon)$, with failure probabilty $\delta$: 
	\begin{align}
	\Prob_{x}[(1-\epsilon)< \|Mx\|_2^2 <(1+\epsilon)]>1-\delta. \nonumber
	\end{align}
	We use the results of \citet{kane2014sparser} who show that with high probability a matrix $P\in \{0,\pm1\}^{n\times d}$ where every row has $p$ non-zero entries, chosen uniformly and independently at random and the non-zero entries have uniformly and independently distributed signs, preserves pairwise distances to within a factor $\epsilon$ for $d=O(\epsilon^{-2}\log(1/\delta))$ and $p=\Theta(\epsilon^{-1}\log(1/\delta))$.
	
	It is easy to verify that inner-products are preserved to within an additive error $\epsilon$ if the pairwise distances are preserved to within a factors of $(1\pm\epsilon)$. By choosing $\delta=1/r^3$ and doing a union bound over all the $r^2$ pairs of factors, the factors are $(\mu\pm \epsilon)$ incoherent in the projected space with high probability if they were $\mu$ incoherent in the original space. Setting $\epsilon=r^{-2}\log^{-1}r$ ensures that $\mu+\epsilon=o(r^{-2})$. Claim 2 now again follows from Theorem 3 of \citet{sharan2017orthogonalized}.
\end{proof}

\end{document}